\documentclass[Afour,sageh,times]{sagej}

\usepackage{moreverb,url}

\usepackage[colorlinks,bookmarksopen,bookmarksnumbered,citecolor=blue,urlcolor=red]{hyperref}

\newcommand\BibTeX{{\rmfamily B\kern-.05em \textsc{i\kern-.025em b}\kern-.08em
T\kern-.1667em\lower.7ex\hbox{E}\kern-.125emX}}
\newcommand{\bx}{{\bf x}}

\newcommand{\bn}{{\bf n}}

\usepackage{caption}
\usepackage{subcaption}

\usepackage{tikz}
\newtheorem{theorem}{\bf Theorem}

\newtheorem{remark}{\bf Remark}

\newtheorem{lemma}{\bf Lemma}
\newtheorem{assumption}{\bf Assumption}

\usepackage[ruled, vlined, linesnumbered]{algorithm2e}
\usepackage{dirtytalk}
\usepackage{bm}

\DeclareMathOperator*{\argmin}{argmin}
\renewcommand{\qed}{\hfill\blacksquare}

\begin{document}

\runninghead{Zinage, Pedram, and Tanaka}
\title{Optimal Sampling-based Motion Planning in Gaussian Belief Space for Minimum Sensing Navigation}

\author{Vrushabh Zinage\affilnum{1}, Ali Reza Pedram\affilnum{2}, and Takashi Tanaka\affilnum{1}}
\affiliation{
\affilnum{1}Department of Aerospace Engineering and Engineering Mechanics, University of Texas at Austin.  \\
\affilnum{2}Walker Department of Mechanical Engineering, University of Texas at Austin.}

\corrauth{Takashi Tanaka, Department of Aerospace Engineering and Engineering Mechanics, University of Texas at Austin, Austin, TX, 78712, USA.}

\email{ttanaka@utexas.edu}

\begin{abstract}
In this paper, we consider the motion planning problem in Gaussian belief space for minimum sensing navigation. Despite the extensive use of sampling-based algorithms and their rigorous analysis in the deterministic setting, there has been little formal analysis of the quality of their solutions returned by sampling algorithms in Gaussian belief space. This paper aims to address this lack of research by examining the asymptotic behavior of the cost of solutions obtained from Gaussian belief space based sampling algorithms as the number of samples increases. To that end, we propose a sampling based motion planning algorithm termed Information Geometric PRM* (IG-PRM*) for generating feasible paths that minimize a weighted sum of the Euclidean and an information-theoretic cost and show that the cost of the solution that is returned is guaranteed to approach the global optimum in the limit of large number of samples. Finally, we consider an obstacle-free scenario and compute the optimal solution using the "move and sense" strategy in literature. We then verify that the cost returned by our proposed algorithm converges to this optimal solution as the number of samples increases.
\end{abstract}


\maketitle

\section{Introduction}
Over the past few decades, motion planning has been an active area of research in the robotics community. Motion planning can be broadly classified into three main categories, mainly grid based \cite{ding2019efficient_hkust_search,MacAllister}, optimization based \cite{mercy2017spline_tcst,Kushleyev,Deits} and sampling based algorithms \cite{karaman2011sampling}. Sampling-based planning algorithms sample points in the configuration space and check whether a connection is possible between the sampled points using collision-checking modules which are treated as a black box and are usually computationally expensive. They have shown success in handling high dimensional motion planning problems for instance manipulators \cite{khan2020control_rrt_manipulator}, warehouse robots, aerial robotics \cite{lee2016planning_rrt_aerial} etc. In the field of robotic motion planning, it is common to first generate a reference path (usually using sampling based algorithms \cite{karaman2011sampling,hart1968formal_astar,kavraki1996probabilistic_prm,lavalle2001rapidly_rrt,janson2015fast_fmt,zinage2020generalized}) and then separately design a feedback control system for following that path. While this two-step approach is not optimal in general, it simplifies the overall problem and is acceptable in many cases. Additionally, this two-step method can take advantage of powerful geometry-based trajectory generation algorithms (such as \cite{hkust,upenn,mellinger,ding2019safe_hkust_optimization,zinage20233d}), which allows for addressing other considerations such as dynamic constraints, randomness, and uncertainty during the control design phase. Despite these benefits, it is important to acknowledge the inherent challenges posed by measurement noise and uncertainty in robotics.

In light of these fundamental and inevitable issues, a different perspective may be more advantageous when planning under uncertainty. Instead of using deterministic boolean indicators to check for collisions in feasible paths, as is done in planning without uncertainty, it is preferable to represent the state of the agent as a probability distribution over the states in the configuration space. 
These states are known as the belief state or the information state. Planning under uncertainty can be formulated as a Partially Observable Markov Decision Process (POMDP) \cite{astrom1965optimal_pomdp_1,smallwood1973optimal_pomdp_2,kaelbling1998planning_pomdp_3}.
In general, for real-world problems, it becomes computationally intractable to compute the solution of these POMDPs despite the recent progress to solve these POMDPs. Planning in belief spaces that are infinite-dimensional becomes more tractable from a generated roadmap using sampling based motion planning algorithms. Therefore, sampling-based algorithms are considered more advantageous for these infinite-dimensional motion planning problems.

The main motivation for this paper arises from the need for simultaneous perception and planning in autonomous systems to address the problem of finding the shortest path while minimizing sensory resource consumption. Advanced and affordable sensing devices have made it easier to gather sensor data, but this may not always be efficient for resource-limited robots due to the significant power and computational resources required. As the number of sensor modalities increases, it is crucial to minimize the use of sensory resources while ensuring that the algorithm converges to an optimal solution rather than a suboptimal one, as the latter can lead to higher costs.
In this paper, we propose the Information Geometric PRM* (IG-PRM*), a sampling-based planning algorithm for minimum sensing navigation that computes a reference path in Gaussian belief space, minimizing the weighted sum of Euclidean and perception costs, with the latter being a quasi-pseudo metric. Our main contribution is the proof of asymptotic optimality for IG-PRM*. The asymptotic optimality of a related algorithm termed IG-RRT* was already conjectured in \cite{pedram2021gaussian}. Asymptotic optimality is essential in motion planning, as it guarantees that the algorithm will eventually find a solution as close as desired to the optimal one, without getting stuck in a suboptimal solution. Following \cite{pedram2021rationally,pedram2021gaussian}, our goal is to develop an asymptotic optimal path planning methodology that enables a robot to navigate using the minimum amount of sensing resources while adapting to any constraints on its sensory resources.

\section{Related work}

\subsection{Belief space motion planning}

In the continuous state, control, and observation space, the complexity of the POMDP framework makes it difficult to use. 
Towards the aim of addressing this issue, the area of belief space path planning computes feasible paths for uncertain systems that take into account the uncertainty in the system's dynamics and environment. \cite{censi2008bayesian_censi} proposed a planning algorithm that uses graph search and constraint propagation on a grid-based representation of the space. \cite{platt2010belief_platt_2010} used nonlinear optimization methods to find the best nominal path in continuous space. The linear quadratic Gaussian motion planning (LQGMP) method \cite{van2011lqg_van_2010} computes the best possible feasible path among a finite number of paths generated by RRT by simulating the performance of an LQG controller on all of them. \cite{bry2011rapidly_bry_and_roy} utilized a tree-based approach to optimize the underlying nominal trajectory using RRT*. \cite{vitus2011closed_related_to_cc_3} also addressed the optimization of the underlying trajectory by formulating it as a chance-constrained optimal control problem. 
\cite{prentice2009belief_brm}  also used roadmap-based methods based on the Probabilistic Roadmap (PRM) approach, where the optimal path is found through a breadth-first search on the belief roadmap (BRM). However, in all of these roadmap-based methods, the optimal substructure assumption is violated, meaning that the costs of different edges on the graph depend on one another. In \cite{kurniawati2012global}, the point-based POMDP (Partially Observable Markov Decision Process) planner takes into account uncertainties in motion, observation, and mapping, and improves upon previous point-based methods through the use of guided cluster sampling. This method starts with a roadmap in the configuration space and grows a single-query tree in the belief space, rooted in the initial belief. \cite{roy1999coastal} examines the use of a coastal navigation strategy to assist the agent's perception during navigation. 



\subsection{Chance constrained motion planning}
Chance constrained motion planning is a class of planning algorithms that take into consideration probabilistic safety constraints. This approach is closely related to the field of belief space path planning, which involves finding paths that meet specific safety requirements under uncertainty \cite{blackmore2006probabilistic_related_to_cc_1,blackmore2011chance_related_to_cc_2,vitus2011closed_related_to_cc_3}. There have been efforts to extend basic chance-constrained methods, which are primarily designed for linear-Gaussian systems, to handle more complex, non-linear, and non-Gaussian problems, as well as to address joint chance constraints \cite{blackmore2010probabilistic_cc_extend_1,wang2020non_cc_extend_2,ono2015chance_cc_extend_3}. However, these formulations can be computationally expensive and may not scale well. To address this issue, several methods were proposed to improve scalability, while others have developed sampling-based approaches like CC-RRT (chance-constrained rapidly-exploring random tree) that allow for efficient computation of feasible paths. The CC-RRT algorithm has been generalized for use in dynamic environments, and several variants have been introduced that guarantee convergence to the optimal trajectory. 
Tree-based planners with chance constraints have been shown to be efficient and have been used in unknown environments through iterative planning in several studies \cite{luders2010chance_unknown_1,pairet2021online_unknown_2,plaku2010motion_unknown_3}. Some studies have also extended these frameworks to systems with measurement uncertainty, using maximum-likelihood observations to approximate solutions \cite{platt2010belief_platt_2010}, although this approach does not provide safety guarantees. 
\subsection{Information theoretic path planning}
In an objective function of belief space path planning, an information-theoretic cost is a measure of uncertainty in the system, which is usually quantified using mutual information. 
For Gaussian distributions, calculation of the cost usually involves computing the determinant of a posteriori covariance matrix, which has a complexity of $O(n^3)$ in general cases, where $n$ is the state dimension. This computation needs to be performed for each potential action. In \cite{indelman2015planning_inf1}, a method was proposed to address this challenge by using information form and exploiting sparsity, but this still requires expensive access to marginal probability distributions. The rAMDL approach  \cite{kopitkov2017no_inf2} performs a one-time calculation of the marginal covariances of the variables involved in the candidate actions, and then uses an augmented matrix determinant lemma (AMDL) to efficiently evaluate the information-theoretic cost for each action. However, this method still requires the recovery of appropriate marginal covariances, the complexity of which depends on the state dimensionality and the sparsity of the system. \cite{levine2013information_inf3} focused on finding ways for sensing agents to both gather as much information as possible about their target or environment, as measured by the Fisher information matrix, and minimize the cost of reaching their goal. In \cite{folsom2021scalable_inf4}, a Mars helicopter used an RRT*-IT algorithm to explore the surface of Mars and reduce uncertainty about the terrain type in the shortest amount of time. 

\subsection{Technical contributions}
The technical contributions of this paper are as follows.
\begin{enumerate}
    \item We consider the shortest path problem for minimum sensing navigation in Gaussian belief space with respect to a cost function that represents the weighted sum of the Euclidean and the information-theoretic sensing cost and propose a sampling-based motion planning algorithm termed Information Geometric PRM* (IG-PRM*) algorithm. 
    \item We prove that the IG-PRM* algorithm is asymptotically optimal. One of the main challenges in establishing asymptotic optimality in Gaussian belief space as compared to deterministic space is the characterization of the volume of covariances and the computation of the associated probability of sampling these covariance matrices. To address this challenge, our second contribution is based on the following two novel results. First, we provide an analytical expression for computing the volume of covariance matrices in the space of symmetric matrices equipped with Rao-Fisher metric and derive a lower bound for this volume in terms of the Selberg integral. Second, we provide a lower bound for the probability of sampling these covariance matrices. These two results aid in establishing the asymptotic optimality of IG-PRM*.
    To the best knowledge of the authors, this is the first paper that addresses the problem of asymptotic optimality in Gaussian belief space.
    \item Through numerical simulations, we verify our claim that the cost returned by IG-PRM* in the absence of any obstacle converges to the optimal cost in the limit of large number of samples. The optimal cost can be analytically computed using the \say{move and sense} strategy in the literature \cite{pedram2021gaussian}.
\end{enumerate}
\subsection{Outline of the paper}
The paper is organized as follows. Section~\ref{sec:notation} discusses the nomenclature followed by preliminaries in Section \ref{sec:prelim}. Section \ref{sec:problem_formulation} discusses the problem statement we address in the paper followed by the proposed algorithm and main result in Sections~\ref{sec:algorithm} and \ref{sec:optimality} respectively. In Section \ref{sec:lossless_ri_prm_star}, we extend the proposed algorithm. Finally, Section \ref{sec:experiments} discusses the numerical simulations and verify the claims made in the paper followed by some concluding remarks in Section~\ref{sec:conclusion}.

\subsection{Notation and convention} \label{sec:notation}
Matrices and vectors are represented by uppercase and lowercase letter respectively. The following notation will be used. $\mathbb{S}^d=\left\{P\in\mathbb{R}^{d \times d}: P \text{ is symmetric.} \right\}$, $\mathbb{S}_{+}^d=\left\{P\in\mathbb{R}^{d \times d}: P\succ 0 \right\}$ and $\mathbb{S}_\rho^d=\left\{P\in\mathbb{S}^d: P\succeq \rho I,\;\rho>0 \right\}$. $\mathcal{E}_{\chi^2}(x_0, P_0)=\{x\in \mathbb{R}^d: (x-x_0)^\top P_0^{-1}(x-x_0)< \chi^2 \}$ is the confidence ellipse. Throughout this paper, we assume that the value of $\chi^2$ is fixed. $\mathcal{B}(x_0,r_0)=\{x\in\mathbb{R}^d:(x-x_0)^\mathrm{T}(x-x_0)\leq r_0\}$ is the ball with center at $x_0$ and radius $r_0$. $\lambda_i(X)$ for all $i\in\{1,2,\dots, d\}$ denote the eigenvalues of a positive definite matrix $X$ and without loss of generality, we assume that $\lambda_1\leq \lambda_2\dots\leq\lambda_d$. For integers $a$ and $b(\geq a)$, $[a;b]$ denotes the set $\{a,a+1,\dots,b\}$. $\mathcal{N}(x,P)$ represents a Gaussian random variable with mean $x$ and covariance $P$. $\Bar{\sigma}(M)$ denotes the maximum singular value of $M$. The Euclidean and the Frobenius norm are represented by $\|.\|$ and $\|.\|_F$, respectively. 
\section{Preliminaries\label{sec:prelim}}
This section provides a brief overview of our setup, which closely mirrors that of \cite{pedram2021gaussian}. We include this information for the sake of completeness.

\subsection{Dynamic model}

Consider a sequence of way points $\{x_k\}_{k\in[1;K]}$ in the configuration space $\mathbb{R}^d$.
Let $t_k$ be the time that the agent is scheduled to visit the $k^\text{th}$ way point $x_k$.
The agent is assumed to apply a constant velocity input
\[
v(t) = v_k := \frac{x_{k+1}-x_k}{t_{k+1}-t_k}
\]
for $t_k \leq t < t_{k+1}, k\in[1;K-1]$.
We assume that the agent motion is subject to stochastic disturbance.
Let $\bx(t_k)$ be the random vector representing the robot's actual position at time $t_k$. It is assumed to satisfy
\begin{align}
 \label{eq:dynamics}
\bx(t_{k+1})=\bx(t_k)+(t_{k+1}-t_k)v_k+\bn_k
\end{align}
where $\bn_k \sim\mathcal{N}(0, \|x_{k+1}-x_k\|W)$.
Note that the constant velocity input and dynamics \eqref{eq:dynamics} are assumed exclusively for the purpose of establishing a distance concept within a Gaussian belief space as presented in subsequent sections. Despite potential significant differences between actual robot dynamics and \eqref{eq:dynamics}, the algorithm we propose in Section \ref{sec:problem_formulation} remains applicable. This approach is analogous to the widespread use of RRT* (or other sampling based motion planning algorithms) for Euclidean distance minimization in scenarios where Euclidean distance in configuration space fails to accurately represent motion costs. In line with this reasoning, we deliberately adopt a simplistic model \eqref{eq:dynamics} and reserve addressing more accurate dynamic constraints for the path-following control phase.
\subsection{Gaussian belief space}

Let the probability distributions of the robot position at time step $k$ can be characterized by a Gaussian model as $\bx_k\sim\mathcal{N}(x_k, P_k)$, where $x_k\in\mathbb{R}^d$ is the mean position and $P_k\in \mathbb{S}_{+}^d$ is the covariance matrix. 
We first introduce an appropriate directed distance function from a point $(x_k, P_k)\in \mathbb{R}^d \times \mathbb{S}_{+}^d$  to another $(x_{k+1}, P_{k+1})\in \mathbb{R}^d \times \mathbb{S}_{+}^d$. 
The distance function is interpreted as the cost of steering the state random variable $\bx_k\sim\mathcal{N}(x_k, P_k)$ at time $k$ to $\bx_{k+1}\sim\mathcal{N}(x_{k+1}, P_{k+1})$ in the next time step under the dynamics \eqref{eq:dynamics}. 
We assume that the distance function is a weighted sum of the travel cost $\mathcal{D}_{\text{travel}}(k)$ and the information cost $\mathcal{D}_{\text{info}}(k)$.

\subsubsection{Travel cost:}
We assume that the travel cost is simply the commanded travel distance:
\[
\mathcal{D}_{\text{travel}}(k):=\|x_{k+1}-x_k\|
\]

\subsubsection{Information cost:}
We define the information-theoretic cost function $\mathcal{D}_{\text{info}}(k)$ at time step $k$
as follows:
\begin{subequations}
\label{eq:d_info_general0}
\begin{align}
\mathcal{D}_{\text{info}}(k)=\!\!\min_{Q_{k+1}\succeq 0}\!\! & \quad \frac{1}{2}\log\det \hat{P}_{k+1}-\frac{1}{2}\log\det Q_{k+1} \label{eq:d_info_general}\\
\text{s.t.} &\quad Q_{k+1} \preceq P_{k+1}, \;\; Q_{k+1} \preceq \hat{P}_{k+1}.\label{eq:d_info_general1}
\end{align}
\end{subequations}
Notice that for any given pair of the origin $(x_k, P_k)$ and the destination $(x_{k+1}, P_{k+1})$, \eqref{eq:d_info_general} takes a nonnegative value. For more details and motivation for using this cost function, we suggest the readers refer to \cite{pedram2021gaussian}.



\subsubsection{Total cost:}
The total cost to steer the state random variable $\bx_k\sim\mathcal{N}(x_k, P_k)$ to $\bx_{k+1}\sim\mathcal{N}(x_{k+1}, P_{k+1})$ is a weighted sum of $\mathcal{D}_{\text{travel}}(k)$ and $\mathcal{D}_{\text{info}}(k)$ as
\begin{align}
\label{eq:def_D}
&\mathcal{D}(x_k, x_{k+1}, P_k, P_{k+1}):= \; \mathcal{D}_{\text{travel}}(k)+\alpha \mathcal{D}_{\text{info}}(k),
\end{align}
where $\alpha\geq 0$ is the weight factor.
\subsection{Chains and Paths}
Suppose that a sequence $\{(x_k, P_k)\}_{k\in[1;K-1]}$ is given. In what follows, the sequence of transitions from $(x_k, P_k)$ to $(x_{k+1}, P_{k+1})$, ${k\in[1;K-1]}$ will be referred to as a \emph{chain}. 
Notice that the parameter \say{$t$} does not necessarily correspond to the physical time. The time of arrival of the agent at the end point depends on the length of the path and the travel speed of the robot. 

\subsubsection{Lossless Chains and Paths:}
A transition from $(x_k, P_k)$ to $(x_{k+1}, P_{k+1})$ is said to be \emph{lossless} if
\begin{equation}
P_{k+1} \preceq \hat{P}_{k+1}(:=P_k+\|x_{k+1}-x_k\|W)
\end{equation}
If every transition in the sequence $\{(x_k, P_k)\}_{k\in[1;K]}$ is lossless, we say that the sequence is lossless.
Let $\gamma: [0,T]\rightarrow \mathbb{R}^d\times \mathbb{S}_{+}^d$, $\gamma(t)=(x(t), P(t))$ be a path. In what follows, the variable $t$ is called the time parameter.
The {travel length} of the path $\gamma$ from time $t=t_a$ to time $t=t_b$ is defined as
\[
\ell(\gamma_x[t_a, t_b])=\sup_{\mathcal{P}} \sum_{k=1}^K \|x(t_k)-x_{k+1}\|
\]
where the supremum is over the space of all partitions $\mathcal{P}=(t_a=t_0<t_1<\cdots < t_K=t_b)$. We say that a path $\gamma$ is \emph{lossless} if the condition
\begin{align}
P(t_b) \preceq P(t_a)+\ell(\gamma_x[t_a, t_b])W
\label{eqn:P_growth}
\end{align}
for any $0\leq t_a < t_b \leq T$. Notice that the right hand side of \eqref{eqn:P_growth} is the covariance of the position of the robot which had the initial configuration $(x(t_a),P(t_a)$ and traveled along the path $\gamma_x[t_a, t_b]$ under the dynamics \eqref{eq:dynamics} without sensing. Therefore, the inequality \eqref{eqn:P_growth} means that the growth of uncertainty from $P(t_a)$ to $P(t_b)$ is never greater that the "natural growth" $\ell W$, which is always true in realistic navigation scenarios. A path $\gamma$ is said to be \emph{finitely lossless} if there exists a finite $N$ and a partition $\mathcal{P}=(0=t_0<t_1<\cdots < t_K=T)$ such that for each $k\in[1;K]$, the transition from $(x(t_k), P(t_k))$ to $(x'(t_k), P'(t_k))$ is lossless

\begin{remark}
\normalfont If a path $\gamma$ is finitely lossless with respect to a partition $\mathcal{P}$, then it is also finitely lossless with respect to a partition $\mathcal{P}'$, provided $\mathcal{P}' \supseteq \mathcal{P}$ (i.e., $\mathcal{P}'$ is a refinement of $\mathcal{P}$). Based on this observation, it can be shown that if a path is finitely lossless then it is lossless. However, the converse is not always true.
\end{remark}

\subsubsection{Collision-free Chains and Paths:}
Let $\mathcal{X}_{\text{obs}}\subset \mathbb{R}^d$ be a closed subset representing obstacles. 
Consider a transition from $x_k$ to $x_{k+1}$. The robot's mean position during this transition is parameterized as
\[
x(\lambda)=(1-\lambda)x_k+\lambda x_{k+1}\;\;\forall\;\;\lambda\in[0,1].
\]
Assuming that the initial covariance is $P_k$, the evolution of the covariance matrix is written as
\[
P(\lambda)=P_k+\lambda\|x_{k+1}-x_k\|W\;\;\forall\;\;\lambda\in[0,1].
\]
For a fixed confidence level parameter $\chi^2>0$, we say that the transition from $(x(0), P(0))$ to $(x(1), P(1))$ is \emph{collision-free} if
\begin{align} \nonumber
&(x(\lambda)-x_{\text{obs}})^\top P(\lambda)^{-1}(x(\lambda)-x_{\text{obs}}) \geq \chi^2,\\
&\quad \quad \quad \forall\; \lambda\in [0,1], \quad \forall x_{\text{obs}}\in \mathcal{X}_{\text{obs}}.\nonumber
\end{align}
\begin{remark}
\normalfont A collision is detected when 
\begin{equation}
(x(\lambda)-x_{\text{obs}})^\top P(\lambda)^{-1}(x(\lambda)-x_{\text{obs}}) < \chi^2\quad\nonumber
\end{equation}
for all $ \lambda\in [0,1]$ and $x_{\text{obs}}\in \mathcal{X}_{\text{obs}}$. The process of detecting collisions can be thought of as a problem of determining whether a particular set of conditions is feasible which is a convex program for each convex obstacle $\mathcal{X}_{\text{obs}}$ \cite{pedram2021gaussian}.
\begin{align} 
\nonumber
&\begin{bmatrix}
\chi^2 &  (1\!-\!\lambda)x_k^\top\!+\!\lambda x_{k+1}^\top\!-\!x_{\text{obs}}^\top \\
(1\!-\!\lambda)x_k\!+\!\lambda x_{k+1}\!-\!x_{\text{obs}} & P_k+\lambda\|x_{k+1}-x_k\|W
\end{bmatrix}\succ 0,\\ \label{eq:col_free}
& \quad \quad \quad \quad \quad   0\leq \lambda \leq 1,  \quad x_{\text{obs}}\in \mathcal{X}_{\text{obs}} 
\end{align}
\label{remark:collision_free}
\end{remark}
We say that a chain $\{(x_k, P_k)\}_{k\in[1;K]}$ is {collision-free} if for each $k\in[1;K-1]$, the transition from $x_k$ to $x_{k+1}$ with the initial covariance $P_k$ is collision-free. We say that a path $\gamma: [0,1]\rightarrow \mathbb{R}^d\times \mathbb{S}_{+}^d$, $\gamma(t)=(x(t), P(t))$ is {collision-free} if
\begin{align} \nonumber
&(x(t)-x_{\text{obs}})^\top P^{-1}(t)(x(t)-x_{\text{obs}}) \geq \chi^2,\\
& \quad \quad \quad \forall\; t\in [0, 1], \quad \forall\; x_{\text{obs}}\in \mathcal{X}_{\text{obs}}.
\end{align}

\section{Problem Formulation}
\label{sec:problem_formulation}
\subsection{Path length}
Let $\gamma: [0,1]\rightarrow \mathbb{R}^d\times \mathbb{S}_{+}^d$, $\gamma(t)=(x(t), P(t))$ be a path, and $\mathcal{P}=(0=t_0<t_1<\cdots < t_{K_n}=1)$ be a partition.
The length of the path $\gamma$ with respect to the partition $\mathcal{P}$ is defined as
\begin{equation}
c(\gamma;\mathcal{P})=\sum_{k=1}^{K_n-1} \mathcal{D}(x(t_k), x(t_{k+1}), P(t_k), P(t_{k+1}))
\end{equation}
The length of a path $\gamma$ is defined as the supremum of $c(\gamma;\mathcal{P})$ over all partitions
\begin{equation}
\label{eq:def_path_length}
c(\gamma):=\sup_\mathcal{P} c(\gamma;\mathcal{P}).
\end{equation}
The definition \eqref{eq:def_path_length} states that for any path with a finite length, there exists a sequence of partitions
$\{\mathcal{P}_i\}_{i\in\mathbb{N}}$ such that
$\underset{n\rightarrow\infty}{\lim}\; c(\gamma; \mathcal{P}_i) =c(\gamma)$.

\subsection{Topology on the path space}
The space of generalized paths is a vector space on which addition and scalar multiplication are defined as $\left(\gamma_1+\gamma_2\right)(t)=\left(x_1(t)+x_2(t), P_1(t)+P_2(t)\right)$ and $a \gamma(t)=(a x(t), a P(t))$ for $a \in \mathbb{R}$, respectively. Let $\mathcal{P}=\left(0=t_0<t_1<\cdots<t_K=T\right)$ be a partition. The total variation of a generalized path $\gamma$ with respect to $\mathcal{P}$ is defined as $|\gamma|_{\mathrm{TV}}=\underset{{\mathcal{P}}}{\sup}\; \|x(0)\| \bar{\sigma}(W)+\bar{\sigma}(P(0))+\sum_{k=0}^{K-1}\left[\| x\left(t_{k+1}\right)-\right.$ $\left.x\left(t_k\right) \| \bar{\sigma}(W)+\bar{\sigma}\left(P\left(t_{k+1}\right)-P\left(t_k\right)\right)\right]$. 
\subsection{Problem statement}
Given an initial belief state $b_0=\left(x_0, P_0\right) \in \mathbb{B}$ be a given initial belief state, a closed subset $\mathcal{B}_{\text {target }} \subset$ $\mathbb{B}$ representing the desired target belief region, and $X_{\text {obs }}^m \subset \mathbb{R}^d$ be the given obstacle $m \in\{1, \ldots, M\}$ where $M\in\mathbb{N}$. Given a confidence level parameter $\chi^2>0$, the problem is to find the shortest path, and can be formulated  as
\begin{align}
\begin{array}{cl}
&\underset{{\gamma \in \mathcal{B} \mathcal{V}[0, T]}}{\min}\;\; c(\gamma) \\
&\text { s.t. }  \gamma(0)=b_0,\;\;\; \gamma(T) \in \mathcal{B}_{\text {target }} \\
& \left(x(t)-x_{\text {obs }}\right)^{\top} P^{-1}(t)\left(x(t)-x_{\text {obs }}\right) \geq \chi^2 \\
& \forall t \in[0, T], \quad \forall x_{\text {obs }} \in X_{\text {obs }}^m, \quad \forall\; m \in\{1, \ldots, M\} .
\end{array}
\label{eqn:problem_statement}
\end{align}
In addition, the proposed algorithm must guarantee that the generated feasible chain converges to the global optimal cost $c^\star$ as the number of samples tends to infinity (Section \ref{sec:asy_opt} ).
\begin{assumption}
\normalfont    We assume there exists a feasible path $\gamma(t)=(x(t),P(t))$ for {the formulated problem} \eqref{eqn:problem_statement} such that $P(t) \in \mathbb{S}^d_{4\rho}$ and $\textup{Tr}(P(t))\leq R$\; for all $t\in[0,1]$ and $R>0$. 
\end{assumption}
\subsection{Continuity of path cost}

\begin{theorem}
\label{theo:cont}
\normalfont Let $\gamma: [0,1]\rightarrow \mathbb{R}^d\times \mathbb{S}_\rho^d$ and $\gamma': [0,1]\rightarrow \mathbb{R}^d\times \mathbb{S}_\rho^d$ be paths. Suppose  $\gamma \in \mathcal{BV}[0, 1]$ and $\gamma' \in \mathcal{BV}[0, 1]$ and they are both finitely lossless. Then, for each $\epsilon > 0$, there exists $\delta > 0$ such that
\[
|\gamma'-\gamma|_{\text{TV}}\leq \delta \quad \Rightarrow \quad  |c(\gamma')-c(\gamma)| \leq \epsilon.
\]
\end{theorem}
\begin{proof}
Please see \cite[Appendix~D]{pedram2021gaussian}
\end{proof}

\subsection{Asymptotic optimality\label{sec:asy_opt}}
Let $Y_n$ denote the cost of the least cost solution returned by a sampling-based motion planning algorithm in $n$ iterations. We define $c^\star = \inf\{c(\gamma) :\; \gamma\; \text{is a feasible path}\}$. An algorithm is said to be asymptotically optimal if
\begin{align}
P(\{\underset{n\rightarrow\infty}{\limsup}\;Y_n =c^\star\}) = 1 
\end{align}
\section{IG-PRM* Algorithm\label{sec:algorithm}}
In section, we first present the notion of uniformly sampling covariances followed by the IG-PRM* algorithm.
\subsection{Uniform sampling of covariance}

To introduce a sampling-based planning algorithm in a Gaussian belief space, we will need a mechanism that allows us to randomly generate candidate covariance matrices. We use the algorithm proposed in \cite{mittelbach2012sampling} for uniformly sampling covariance $P\in \mathbb{S}^d_{+}$ in  $\mathcal{R}_{(\underline{c},\bar{c}]}:=\{P\in \mathbb{S}^d_{+}: \underline{c} < \textup{Tr}(P)\leq \bar{c}\}$. The sampled covariance $P$ is said to have uniform distribution  on $\mathcal{R}_{[\underline{c},\bar{c}]}$ if
\begin{align} \label{eq:uniform_sampling}
    \mathbb{P}(\{P\in\mathcal{A}\})=\frac{\text{vol}(\mathcal{A}\cap \mathcal{R}_{[\underline{c}, \bar{c} ]})}{\text{vol}(\mathcal{R}_{[\underline{c} ,\bar{c} ]})}
\end{align}
holds for all $\mathcal{A} \subset \mathbb{S}^{d}_+$. We denote by set $\mathcal{R}:=\{P\in \mathbb{S}^d_{+}: \textup{Tr}(P)=1\}$. 
We assume the space of $\mathbb{S}^d_{+}$ is equipped with the Rao-Fisher metric \cite{terras2012harmonic}, and use this metric to measure the volume of different regions in $\mathcal{S}^d_{+}$ in the following Theorem. 
\begin{theorem}
  \normalfont \cite{mittelbach2012sampling} The volume of region $\mathcal{R}_{(c_1,c_2]}$ where $c_2>c_1$ is given by
   \begin{align}
   \label{eq:vol_Rc}
       &\textup{Vol}(\mathcal{R}_{(c_1,c_2]})\nonumber\\
       &=\left(\frac{c_2^{\frac{d(d+1)}{2}}-c_1^{\frac{d(d+1)}{2}}}{\frac{d(d+1)}{2}}\right)V_r\nonumber\\
       &=\left(\frac{c_2^{\frac{d(d+1)}{2}}-c_1^{\frac{d(d+1)}{2}}}{\frac{d(d+1)}{2}}\right)\pi^{\frac{1}{4} d(d-1)} \frac{\prod_{k=2}^d \Gamma\left(\frac{k+1}{2}\right)}{\Gamma\left(\frac{d(d+1)}{2}\right)}
   \end{align}
   where $V_r=\textup{Vol}(\mathcal{R})$ and $\Gamma$ is the gamma function.
\end{theorem}
\begin{proof}
    Please refer to \cite{mittelbach2012sampling}.
\end{proof}
\begin{theorem}
\label{theo:volume}
\normalfont The volume of region $\mathcal{D}_A=\{Q\in \mathbb{S}^d_{+}: Q \preceq A:=\textup{diagonal}(a_1, \dots, a_d)\}$ is lower bounded as
\begin{align}
\label{eq:vol_DA}
    \textup{Vol}(\mathcal{D}_A) \geq V_db^{\frac{d(d+1)}{2}}S_d(1,1,1/2)=\textup{Vol}(\mathcal{D}'_A)
\end{align}
where $b=\underset{k\in[1;d]}{\min}\;a_k$, $\mathcal{D}'_{A}=\{P\in\mathbb{S}^d_+:\;P\preceq b I \}$, $V_d = (d! \; 2^d)^{-1} \frac{2^d \pi^{d^2/2}}{\Gamma_d(d/2)} $ and $S_d(\alpha_1,\alpha_2,\alpha_3)$ is the Selberg integral given by
\begin{align}
     &S_d(\alpha_1, \alpha_2, \alpha_3)\nonumber\\
     &=\int_0^1 .. \int_0^1 \prod_{i=1}^d t_i^{\alpha_1-1}\left(1-t_i\right)^{\alpha_2-1} \prod_{1 \leq i<j \leq d}\left|t_i-t_j\right|^{2 \alpha_3} d t \nonumber \\ 
     & =\prod_{j=0}^{d-1} \frac{\Gamma(\alpha_1+j \alpha_3) \Gamma(\alpha_2+j \gamma) \Gamma(1+(j+1) \gamma)}{\Gamma(\alpha_1+\alpha_2+(d+j-1) \alpha_3) \Gamma(1+\alpha_3)}\nonumber
\end{align}
Further, for $\beta>1$, we have
\begin{align}
    &\textup{Vol}(\mathcal{D}_{\beta A})-\textup{Vol}(\mathcal{D}_{ A})\nonumber\\
    &\geq V_d (\beta^{\frac{d(d+1)}{2}}-1) b^{\frac{d(d+1)}{2}}S_d(1,1,1/2)
\end{align}
where
$\Gamma$ is the gamma function.
\end{theorem}
\begin{proof}
Please, see Appendix~\ref{app:volume}.
\end{proof}

%

\begin{lemma}
\label{lemma:diag}
  \normalfont  If $P\in \mathbb{S}^d_{+}$ has the eigenvalue decomposition $P=V^\top \Sigma V$, where $\Sigma$ is diagonal and $V$ is a unitary matrix, then 
\begin{align}
    \textup{Vol}(\mathcal{D}_P)=\textup{Vol}(\mathcal{D}_{\Sigma}),
\end{align}
where $\mathcal{D}_P=\{Q\in \mathbb{S}^d_{+}: Q \preceq P\}$ and $\mathcal{D}_{\Sigma}=\{Q\in \mathbb{S}^d_{+}: Q \preceq \Sigma\}$.
\end{lemma}
\begin{proof}
{Please, see Appendix~\ref{app:diag}}
\end{proof}
The pseudo code for generating positive definite matrices uniformly from $\mathcal{R}_{(\underline{c},\bar{c}]}$ is described in Algorithms 1-4 of \cite{mittelbach2012sampling}.
\subsection{IG-PRM* Algorithm}
The implementation of PRM* in Gaussian belief space (termed as IG-PRM*)
for the introduced cost function \eqref{eq:def_D} is summarized in Algorithm 1. The source code for Algorithm 1 is available at \href{https://github.com/AlirezaPedram/RI-PRMstar}{ https://github.com/AlirezaPedram/RI-PRMstar}. PRM* \cite{karaman2010incremental} creates a probabilistic roadmap in deterministic configuration space by connecting randomly sampled points, and efficiently finds the shortest path between two points while maintaining asymptotic optimality.
At the first glance, Algorithm~1 seems identical to the original PRM*. However, the implementation in Gaussian belief space necessitates the introduction of new functionalities.

\begin{algorithm}
    { 
        \textbf{Inputs}: $ b_{\text{init}}, \;b_{\text{final}},\;n$ \;
    $\left.B \leftarrow\left\{b_{\text{init}}\right\} \cup\{\text {SampleFree}_{i}\right\}_{i\in [1;n]}\cup\left\{b_{\text{final}}\right\} $\;
    $E \leftarrow \emptyset$\;
    \For{$b\in B$}{
    $\;\; B_{\text{nbors}} \leftarrow \operatorname{Near}\left(B, b, D_{\text{min}}\right)$\;
    \For{$b_j\in B_{\text{nbors}}$}{
    \If{$\operatorname{CollisionFree}(b, b_j, \chi^2)$ }{$E \leftarrow E \cup\{(b, b_j)\}$\;
    }
    }
    }
    $G=(B, E)$\;
    $\gamma_n\leftarrow \;\text{Search}(G, b_{\text{init}}, b_{\text{final}}, k)$\;
    \Return $\gamma_n$
    }
\caption{IG-PRM*}
\label{alg:igpp_body1}
\end{algorithm}
\begin{algorithm}
    {
    \textbf{Inputs}: $G, b_{\text{init}}, b_{\text{final}}, k$ \;
    $N_{b_{\text{init}}} \leftarrow \text{Near}_{\text{from}}(G, b_\text{init},k) $\;
    $N_{b_{\text{final}}} \leftarrow \text{Near}_{\text{to}}(G, b_\text{final},k)$\;
    $\gamma_n\leftarrow \text{ShortestPath}(G, N_{b_{\text{init}}}, N_{b_{\text{final}}})$\;
    \Return $\gamma_n$
    }
\caption{Search Algorithm (\text{Search}($G, b_{\text{init}}, b_{\text{final}}, k$))}
\label{algo:search}
\end{algorithm}

\begin{algorithm}
    { 
     \textbf{Inputs}: $ \gamma_n:=(b_1\rightarrow b_1\dots\rightarrow b_m)$ \;
    \For{$j=2$ \textbf{to} $m$}{
    $ b_j\leftarrow \text{LosslessRefine}(b_{j-1},b_j)$\;
    
    }
    
     \Return $\gamma_n$
    }
\caption{Lossless modification of $\gamma_n$}
\label{alg:lossless_modify_ri_prm_star}
\end{algorithm}
The IG-PRM* (Algorithm \ref{alg:igpp_body1}) takes in as inputs, the initial belief state $b_{\text{init}}$, the final belief state $b_{\text{final}}$ and the total number of belief states $n$ that would be sampled. The function $\textup{SampleFree}_i$ (Line 2 of Algorithm \ref{alg:igpp_body1}) generates a belief state $b_i=(x_i,P_i)$ by sampling the mean state $x_i$ and the corresponding covariance $P_i$ independently. The point  
$x_i$ is sampled uniformly from obstacle free space $\mathcal{X}_{\text{free}}\in \mathbb{R}^d$ and $P_i$ is sampled uniformly from $\mathcal{R}_{(\rho d, R ]}:=\{P \in \mathbb{S}^d_{+}:\; \rho d < \textup{Tr}(P) \leq R \}$ by the scheme proposed in \cite{mittelbach2012sampling}.  Using the function $\textup{SampleFree}_i$, $n$ such uniformly randomly belief states. These $n$ sampled belief states along with the $b_{\text{init}}$ and $b_{\text{final}}$ are stored in set $B$ (Line 2 of Algorithm 1). For every belief state $b$ in the set $B$, the function $\operatorname{Near}\left( B, b, D_{\text{min}}\right)$ (Line 5 of Algorithm \ref{alg:igpp_body1}) returns the neighboring nodes of $b$ in $B$ and stores it in set $B_{\text{nbors}}$. In other words,  $B_{\text{nbors}}=\left\{b_j \in B:\hat{\mathcal{D}}(b_j,b)\leq  D_{\text{min}}  \right\}$, 
where 

\begin{figure*}[ht!]
\centering
\includegraphics[width=1.96\columnwidth]{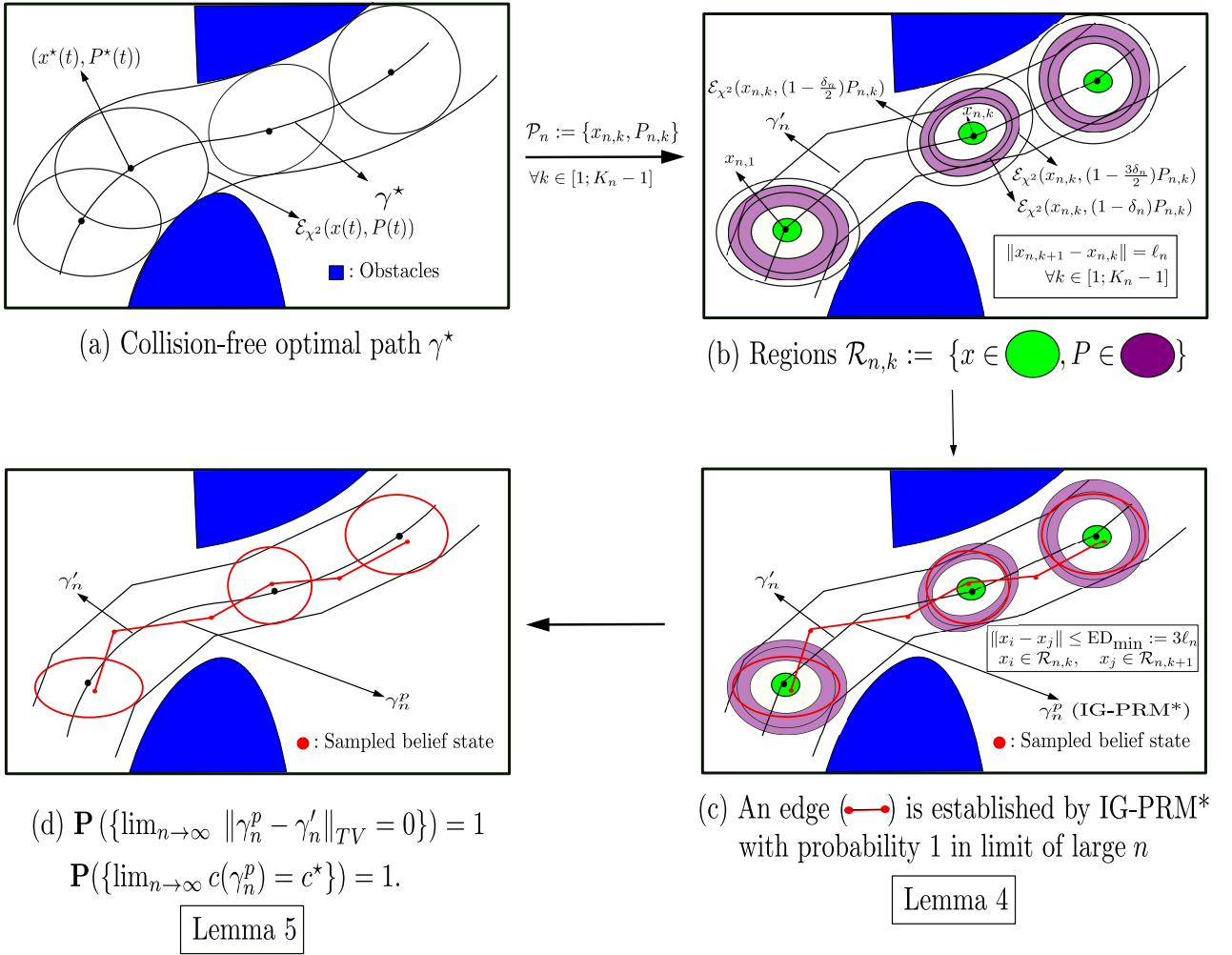}
\caption{The figure depicts the outline of the proof. (a) Representation of optimal path $\gamma^\star$ with cost $c(\gamma^\star)=c^\star$. (c) The zig-zag chain denoted by $\gamma^p_n$ (red curve) is the chain generated by the IG-PRM* algorithm (Algorithm \ref{alg:igpp_body1}) where $(x_i,P_i)$ denote the belief states that are sampled by IG-PRM*. In Lemma \ref{lemma:event_E_n}, we show that an edge is established by IG-PRM* with probability 1 in the limit of large $n$. (d) We show in Lemma \ref{lemma:arbritrarily_close} that the zig-zag path generated by IG-PRM* converges to the optimal chain $\gamma'_n$ in the limit of large $n$. Using continuity of cost function (Theorem \ref{theo:cont}), we show that the cost of the zig-zag chain $\gamma_n^p$ converges to the optimal cost $c^\star$. }
\label{fig:outline_of_the_proof}
\end{figure*}

\begin{figure*}[ht!]
\centering
\includegraphics[width=1.96\columnwidth]{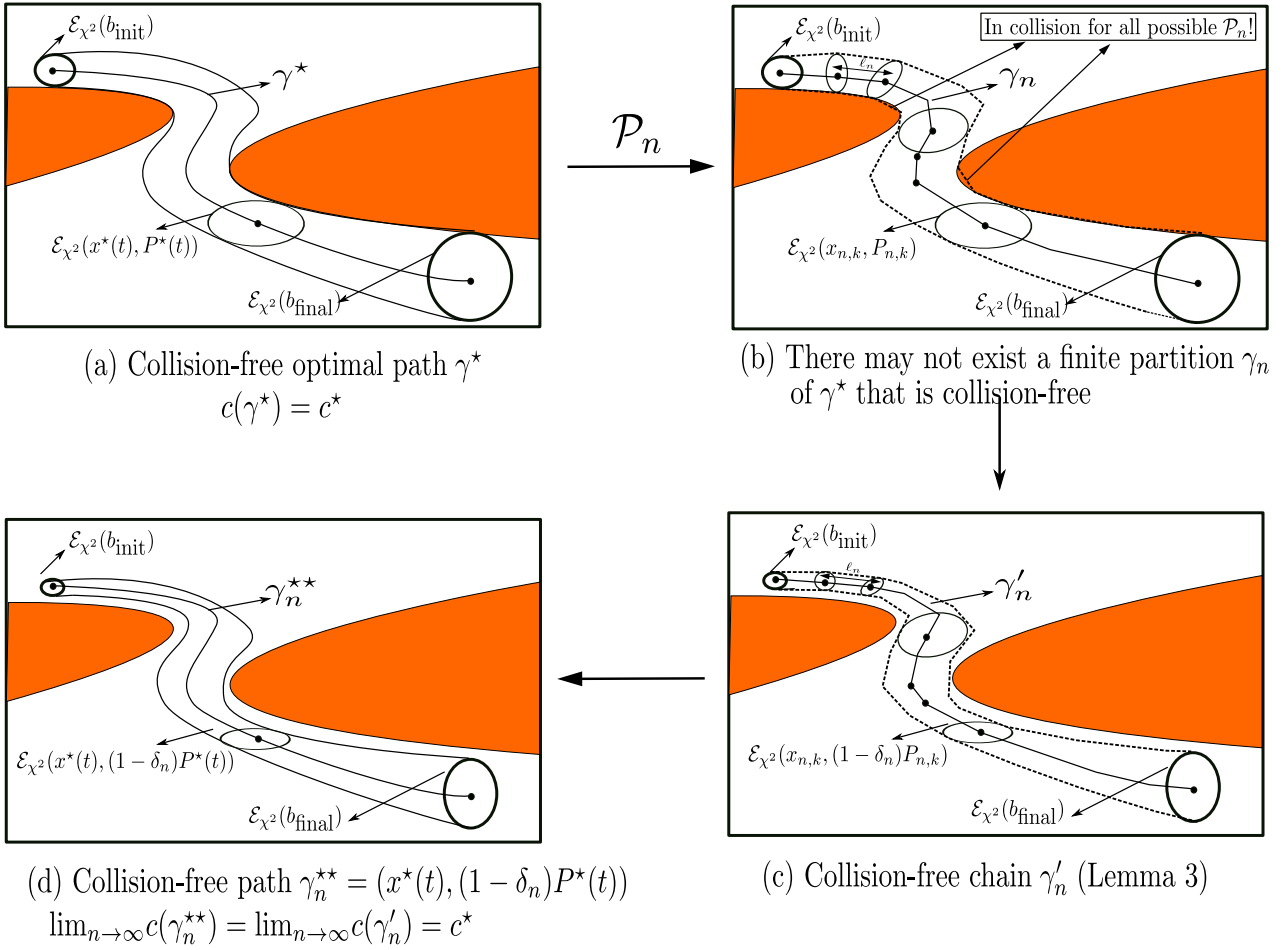}
\caption{Depiction of different paths $\gamma^\star$, $\gamma^{\star\star}_n$ and chains $\gamma_n$, $\gamma'_n$ that are used in the proof of asymptotic optimality of IG-PRM*. (a) The optimal path $\gamma^\star$ is optimal i.e. $c(\gamma^\star)=c^\star$ and collision-free. (b) From Fig. (2b), there does not exist any partition $\mathcal{P}_n$ of the optimal path $\gamma^\star$ that would lead to the chain $\gamma_n:=\{x_{n,k},P_{n,k}\}$ for all $k\in[1;K_n]$ being collision-free (clarified more in Fig. \ref{fig:transition}). To address this issue, we consider a modified chain $\gamma'_n:=\{x_{n,k},P'_{n,k}\}=\{x_{n,k},(1-\delta_n)P_{n,k}\}$ for all $k\in[1;K_n]$ which is collision-free (Lemma \ref{lemma:rn}). (d) We then show that as the number of samples $n$ tends to infinity, the cost of $\gamma'_n$ is equal to the cost of $\gamma^{\star\star}_n$ i.e. $\underset{n\rightarrow\infty}{\lim} c(\gamma^{\star\star}_n)=\underset{n\rightarrow\infty}{\lim} c(\gamma'_n)=c^\star$ (Eqn. \ref{eqn:cost_gamma_dash_equals_gamma}). }
\label{fig:1}
\end{figure*}
\begin{subequations}
\begin{align} 
& \hat{\mathcal{D}}(b,b'):=\|x-x'\|,\\
& D_{\text{min}}=3\ell_n,\\
   & \ell_n:=\min\left\{\frac{\delta_n\sqrt{ \chi^2 \rho} }{16},\frac{\delta_n\rho }{18\bar{\sigma}(W)}\right\},\label{eq:ell_def}\\
   &\delta_n = \min \left\{ \gamma\left(\frac{\text{log}n}{n}\right)^\frac{2}{d(d+3)}, \frac{1}{2}\right\}\\
   &\gamma>\left(\frac{d^2+3d+2}{g_1 g_2 d(d+3)} \right)^\frac{2}{d(d+3)},\\
   &g_1=\frac{\tau_d  h^d}{  \mathcal{V}_{\mathcal{X}}},\;\;\mathcal{V}_\mathcal{X}:=\text{vol}(\mathcal{X}),\\
   & g_2=V_dS_d(1,1,1/2)2^{-1}(d(d+1))V_r^{-1}(2\rho)^{\frac{d(d+1)}{2}}\Delta,\\
   &\Delta=\bigg[R^{\frac{d(d+1)}{2}}- (\rho d)^{\frac{d(d+1)}{2}} \bigg]^{-1},\\
   &h:=\min\left\{\frac{\sqrt{ \chi^2 \rho} }{16},\frac{\rho }{18\bar{\sigma}(W)}\right\} = \frac{\ell_n}{\delta_n}
\end{align}
\end{subequations}

The main intuition behind choosing these geometric constants in a certain way is to allow us to prove the asymptotic optimality and would be more clear in subsequent sections. Now, for every belief state $b_j$ in $B_{\text{nbors}}$, the function $\operatorname{CollisionFree}(b, b_j)$ (Line 7 of Algorithm \ref{alg:igpp_body1}) checks that the $\chi^2$ confidence bound in transition $b \rightarrow b_j $ does not intersect with any obstacles. If the transition $b \rightarrow b_j $ is collision-free, an edge is connected between $b$ and $b_j$ and stored in $E$. Next, a graph $G$ (Line 9 of Algorithm \ref{alg:igpp_body1}) is constructed using the set of belief states $B$ and the edges $E$. Finally, the shortest path $\gamma_n$ (Line 10 of Algorithm \ref{alg:igpp_body1}) between the initial belief $b_{\text{init}}$ to $b_{\text{final}}$ is computed using the Search function (Algorithm \ref{algo:search}). 
Algorithm \ref{algo:search} resembles Algorithm~7 in \cite{choset2005principles}. The function $\text{Near}_{\text{from}}(G, b_\text{init},k)$ finds the $k^\text{th}$ nearest nodes in the metric $\hat{D}$ from $G$, to which the transition from $b_{\text{init}}$ are collision-free. Likewise, the function $\text{Near}_{\text{from}}(G, b_\text{start},k)$ finds the $k^\text{th}$ nearest nodes based on the metric $\hat{D}$ from $G$, from which the transition to $b_{\text{final}}$ is collision-free. The function $\text{ShortestPath}(G, N_{b_{\text{init}}}, N_{b_{\text{final}}})$ first uses  Dijkstra's algorithm \cite{dijkstra1959note} to find the shortest path on $G$ between all possible  pairs of   $b_1 \in N_{b_{\text{init}}}$ and $b_2 \in N_{b_{\text{final}}}$, if one exists. Then, this function returns the path that results in the shortest path, among the sought shortest paths, between $b_{\text{init}}$ and $b_{\text{final}}$, or returns $\text{Failure}$ if  no path is found.  However, it must be noted that the belief chain $\gamma_n$ computed using the IG-PRM* algorithm (Algorithm \ref{alg:igpp_body1}) might not be finitely lossless. Algorithm \ref{alg:lossless_modify_ri_prm_star} ensures that the chain $\gamma_n$ is finitely lossless. It takes in as input, $\gamma_n$ which consists of say $m$ belief states in sequence. Then for every belief edge $(b_{j-1},b_j)$ for $j\in[1;m]$, the belief state $b_j$ is updated using the function $\operatorname{LosslessRefine}(b_{j-1},b_j)=(x_j,P^\star)$ which ensures that the transition $b_{j-1}\rightarrow b_j$ is lossless. In other words, $b_j$ now becomes equal to $(x_j,P^\star)$.  Here, $P^\star$ is the minimizer of \eqref{eq:d_info_general} computed using Lemma \ref{lemma:analytical}. In this case $\hat{P}_{k+1}=P_{j-1}+\|x_{j-1}-x_j\|W$ and $P_{k+1}=P_j$ where $b_{j-1}\equiv(x_{j-1},P_{j-1})$ and $b_j\equiv(x_j,P_j)$.
\begin{lemma}
 \normalfont [Lemma 1, \cite{pedram2021gaussian}] Let $[U, \Sigma]$ be the eigendecomposition of $P_{k+1}^{-1 / 2} \hat{P}_{k+1} P_{k+1}^{-1 / 2}$ i.e. $U \Sigma U^{\top}=P_{k+1}^{-1 / 2} \hat{P}_{k+1} P_{k+1}^{-1 / 2}$, where $\Sigma=$ $\operatorname{diag}\left(\sigma_1, \ldots, \sigma_n\right) \succ 0$ and $U$ is unitary matrix. Then, $P^\star=$ $P_{k+1}^{1 / 2} U S^\star U^{\top} P_{k+1}^{1 / 2}$ is the optimal solution of \eqref{eq:d_info_general} , where $S^\star:=$ $\operatorname{diag}\left(\min \left\{1, \sigma_1\right\}, \ldots, \min \left\{1, \sigma_{\mathrm{n}}\right\}\right)$
 \label{lemma:analytical}
\end{lemma}
\begin{assumption}
    \normalfont We assume IG-PRM* is run for $n\geq n^\star$, where $n^\star$ is the minimum $n'\in  \mathbb{N}$ such that $\delta_{n'}< 1/2$.
\end{assumption}
 {
}   

\section{Asymptotic Optimality of IG-PRM* \label{sec:optimality}}
In this section, we show Algorithm~1 with $D_{\text{min}}$ introduced in Section~\ref{sec:algorithm} achieves asymptotic optimality.

\subsection{Outline of the proof\label{subsection:outline_of_proof}}
Let $\gamma^\star: [0,1]\rightarrow \mathbb{R}^d\times \mathbb{S}_\rho^d$, $\gamma^\star(t)=(x^\star(t), P^\star(t))$ be the optimal path and is therefore collision-free by definition (see Fig. (1a)).  We define a series of partitions $\mathcal{P}_n:=\{x_{n,k},P_{n,k}\}\;k\in[1;K_n]$ for $\gamma^{\star}$ where $x_{n,k}=x^\star\left(\frac{k-1}{K_n}\right)$ and $P_{n,k}:=P\left(\frac{k-1}{K_n}\right)$. By construction, $\|x_{n,k+1}-x_{n,k}\|\leq\ell_n$ for all $k\in[1;K_n-1]$ (see Fig. (1b)). 

For each $\mathcal{P}_n$, a sequence of regions $\mathcal{R}_{n,k}, \;k\in[1;K_n]$ is constructed (Section \ref{sec:construction_free_rnk}) which posses three important properties. First, each $\mathcal{R}_{n,k}$ is a region of finite volume. Second, in the limit of large $n$, the volume of these regions converges to zero and the set $\mathcal{R}_{n,k}$ converges to the point $(x_{n,k},P_{n,k})$ for all $k\in[1;K_n]$. Last, we show that the transition between consecutive regions i.e., between any belief state  $b_k \in \mathcal{R}_{n,k}$ and any belief state $b_{k+1} \in \mathcal{R}_{n,k+1}$ is collision-free, and the distance between $b_k$ and $b_{k+1}$ is smaller than $D_{\textup{min}}$ (Lemma \ref{lemma:rn}). The main intuition behind the construction of these $\mathcal{R}_{n,k}$, is to show that there is a positive probability that a belief state would be sampled in every $\mathcal{R}_{n,k}$ for $k\in[1;K_n]$ by the IG-PRM* algorithm (Algorithm \ref{alg:igpp_body1}) and this would be more clear in subsequent sections. Now, we claim that if the IG-PRM* algorithm is able to sample belief states in every $\mathcal{R}_{n,k}$, the connection between any $b_k \in \mathcal{R}_{n,k}$ and $b_{k+1} \in \mathcal{R}_{n,k+1}$ will be established by Algorithm~\ref{alg:igpp_body1}, if they are sampled (i.e. if $b_k \in B$ and $b_{k+1} \in B$) as shown in Fig. (1c). To prove this mathematically, we define the event $E_{n} \triangleq E_{n,1}\cap E_{n,2}\dots \cap E_{n,K_n}$ (Section \ref{subsec:event_E}) as the event that a belief state is sampled inside all $\mathcal{R}_{n,k}$ regions, and show that the event $E_n$ occurs with probability one as $n$ tends to infinity (Lemma \ref{lemma:event_E_n}).

Note that by definition, the optimal path $\gamma^\star$ is collision-free. However, the partition chain $\mathcal{P}_n$ might not be collision-free as one of the confidence ellipsoids of $\mathcal{P}_n$ would pass along the boundary of the obstacle as shown in Fig. \ref{fig:1}. To fix this issue, we construct a  chain $\gamma_n'\triangleq(x_{n,k},P'_{n,k}:=(1-\delta_n)P_{n,k}),\;k\in[1;K_n]$ (Section \ref{subsec:gamma_def}) such that $(x_{n,k},P'_{n,k}) \in \mathcal{R}_{n,k}$ and thus the transition between $(x_{n,k},P'_{n,k})$ to $(x_{n,k+1},P'_{n,k+1})$ for $k\in[1;K_n-1]$ becomes collision-free. Next, we show that the cost of $\gamma'_n$ converges to the cost of the optimal path $c^\star=c(\gamma^\star)$ in the limit of large $n$. Now, we need to link the cost of collision-free $\gamma'$ to the cost of the feasible paths generated by the IG-PRM* algorithm (Algorithm \ref{alg:igpp_body1}) as we are ultimately interested in analyzing the cost of paths returned by IG-PRM*. To that end, we show there exists a path on the graph generated by the IG-PRM$^\star$ algorithm (more precisely, the path generated by connecting sampled nodes in consecutive $ \mathcal{R}_{n,k+1}$s) that gets arbitrarily close to $\gamma'_n$ as $n$ tends to infinity (Lemma \ref{lemma:arbritrarily_close}). Even though the IG-PRM* generated path would get arbitrarily close to $\gamma'_n$ in the limit of large $n$, it is still not clear whether the cost of these generated paths would tend to $c(\gamma'_n)$ as $n$ tends to infinity. To that end, we leverage the continuity of path cost function (Theorem \ref{theo:cont}) to show that the cost of that path gets arbitrarily close to $c^\star$.



\subsection{Construction of collision-free region \texorpdfstring{$\mathcal{R}_{n,k}$}{R}\label{sec:construction_free_rnk}}
\label{sec:event_d_n}
Let $\{\delta_n\}$ for $n\in\mathbb{N}$ be a sequence of positive numbers defined in Section~\ref{sec:algorithm}. Note that $0< \delta_n \leq \frac{1}{2}$ for each $n$ and $\underset{n\rightarrow\infty}{\lim}\;\delta_n=0$. Let $c^\star$ be the optimal path length and
\[
\ell^\star=\sup_{\mathcal{P}} \sum_{k=1}^{K_n-1}\|x^\star(t_{k+1})-x^\star(t_k)\|
\]
be the total travel length of the optimal path $\gamma^\star:[0,1]\rightarrow \mathbb{R}^d \times \mathbb{S}^d, \gamma^\star(t)=(x^\star(t), P^\star(t))$. 
For each $n\in \mathbb{N}$, choose $\ell_n$ as defined in \eqref{eq:ell_def}
and $K_n=\lceil \frac{\ell^\star}{\ell_n} \rceil$. Consider the equi-spaced partition $\mathcal{P}_n=(0=t_{n,0}\leq t_{n,1} \leq \cdots \leq t_{n,K_n}=1)$, and define the chain $\gamma_n:=(x_{n,k}, P_{n,k})\triangleq (x^\star(t_{n,k}), P^\star(t_{n,k}))$, $k\in[1;K_n]$.
By construction, we have $\|x_{n,k+1}-x_{n,k}\|\leq \ell_n$ for each $k\in[1;K_n-1]$. However,
the chain $\gamma_n$ constructed this way is not collision-free in general (Fig. \ref{fig:1}). To that end, we define regions 
\begin{align}
\nonumber
    \mathcal{R}_{n,k}=\{&(x,P): x\in \mathcal{B}(x_{n,k},\ell_n), \\ \label{eq:R_n_def}
    &\left(1-\frac{3 \delta_n}{2}\right) P_{n,k} \preceq P\preceq \left(1-\frac{\delta_n}{2}\right) P_{n,k}\},
\end{align} 
and show that the transition from any $b_1\equiv(x_1, P_1) \in \mathcal{R}_{n,k}$ to any $b_2\equiv(x_2, P_2) \in \mathcal{R}_{n,k+1}$ is collision-free.
In transition $b_1\rightarrow b_2$, the mean and  covariance of the state can be parameterized as $x[\lambda] = (1-\lambda)x_1+\lambda x_2$ and  $P[\lambda]=P_1+\lambda \|x_2-x_1\|W$ for $\lambda\in[0,1]$.
\begin{remark}\label{remark:semi-axis}
\normalfont The length of minor semi-axis of confidence ellipse $\mathcal{E}_{\chi^2}(x,P)$ is $\sqrt{\chi^2 \underline{\sigma}(P)}$. Consequently, the minimum distance between the boundary of co-centric, similar ellipses   $\mathcal{E}_{\chi^2}(x,P)$ and  $\mathcal{E}_{\chi^2}(x, \alpha P)$, where $\alpha > 0$, is $\sqrt{\chi^2 |1-\alpha|\underline{\sigma}(P)}$. 
\end{remark}

\begin{figure}[ht!]
\centering
\includegraphics[width=0.97\columnwidth]{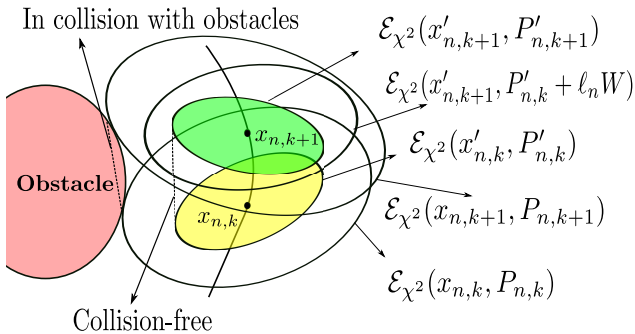}
\caption{Transition from $(x'_{n,k}, P'_{n,k}:=(1-\delta_n)P_{n,k})$ to $(x'_{n,k+1}, P'_{n,k+1}:=(1-\delta_n)P_{n,k+1})$ for all $k\in[1;K_n-1]$ is collision-free even if transition from $(x_{n,k}, P_{n,k})$ to $(x_{n,k+1}, P_{n,k+1})$ is not.}
\label{fig:transition}
\end{figure}

\begin{lemma}
\label{lemma:rn}
 \normalfont In any transition from $b_1=(x_1, P_1) \in \mathcal{R}_{n,k}$ to $b_2=(x_2, P_2) \in \mathcal{R}_{n,k+1}$, 
 for all $\lambda\in[0,1]$,
 $ \mathcal{E}_{\chi^2} (x[\lambda], P[\lambda]) \subset\mathcal{E}_{\chi^2} (x_{n,k},P_{n,k})$, for $    \ell_n:=\min\left\{\frac{\delta_n\sqrt{ \chi^2 \rho} }{16},\frac{\delta_n\rho }{18\bar{\sigma}(W)}\right\}$. This relation proves that the transition $b_1=(x_1,P_1)\rightarrow b_2=(x_2,P_2)$ fully resides in $\mathcal{E}_{\chi^2} (x_{n,k},P_{n,k})$, and thus it is collision-free. 
\\ \label{eq:final_in}

\end{lemma}

\begin{proof}
The proof is provided in two stages. In the first stage, we show 
\begin{align}
\label{eq:intran_P}
    P[\lambda] \preceq \left(1-\frac{\delta_n}{3}\right) P_{n,k}, \quad \forall \lambda \in [0,1].
\end{align}
As the first step to prove \eqref{eq:intran_P}, we note that  
\begin{align*}
 &||x_2-x_1|| \leq \\
 &||x_{n,k}-x_1 ||+ ||x_{n,k+1}- x_{n,k} ||+||x_2-x_{n,k+1}||\\
 &\leq \ell_n + \ell_n + \ell_n  = 3\ell_n=:D_{\text{min}}.  
\end{align*}
Thus, we have
\begin{align} \nonumber
     P[\lambda] &= P_1+ \lambda ||x_2-x_1||W \\
     &\preceq  \left(1-\frac{\delta_n}{2}\right)P_{n,k}+||x_2-x_1||W \\ \nonumber
     &\preceq  \left(1-\frac{\delta_n}{2}\right) P_{n,k}+  3 \ell_n W. 
\end{align}
Hence, to establish \eqref{eq:intran_P}, it suffices to show
\begin{align} \label{eq:R_n_inter}
       \left(1-\frac{\delta_n}{2}\right) P_{n,k}+  3 \ell_n W \preceq  \left(1-\frac{\delta_n}{3}\right) P_{n,k}
\end{align}
Condition \eqref{eq:R_n_inter}
is equivalent to 
\begin{align*}
    3 \ell_n W &\preceq \left[ \left(1-\frac{\delta_n}{3}\right)- \left(1-\frac{\delta_n}{2}\right)\right]  P_{n,k}\\
    & = \frac{\delta_n}{6} P_{n,k},
\end{align*}
which trivially holds as $\ell_n \leq\frac{\delta_n\rho}{18\bar{\sigma}(W)}$. Consequently, 
\begin{align}
&\mathcal{E}_{\chi^2}(x_{n,k}, P[\lambda]) \subseteq \mathcal{E}_{\chi^2} \left(x_{n,k}, \left(1-\frac{\delta_n}{3}\right) P_{n,k}\right),\nonumber
\end{align}
for all $\lambda\in[0,1]$. The minimum distance between $\mathcal{E}_{\chi^2}(x_{n,k}, P_{n,k})$ and $\mathcal{E}_{\chi^2}\left(x_{n,k},  \left(1-\frac{\delta_n}{3}\right) P_{n,k}\right)$ is  
$\left(1-\sqrt{(1-\frac{\delta_n}{3})}\right) \sqrt{ \chi^2 \underline{\sigma}(P_{n,k})}\geq \frac{\delta_n}{8}\sqrt{ \chi^2 \underline{\sigma}(P_{n,k})}$. After linear translating $\mathcal{E}\left(x_{n,k},  \left(1-\frac{\delta_n}{3}\right) P_{n,k}\right)$  for  
$||x[\lambda]-x_{n,k}|| \leq (1-\lambda) \|x_1-x_{n,k}\|+\lambda \|x_2-x_{n,k}\| \leq (1-\lambda)\ell_n +2 \lambda (\|x_2-x_{n,k+1}\|+\|x_{n,k+1}-x_{n,k}\|) \leq (1+\lambda) \ell_n \leq 2\ell_n \leq \frac{\delta_n}{8} \sqrt{\chi^2 \rho} \leq \frac{\delta_n}{8} \sqrt{ \chi^2 \underline{\sigma}(P_{n,k})}$, the resultant ellipse $\mathcal{E}_{\chi^2}\left(x[\lambda], 
\left(1-\frac{\delta_n}{3}\right) P_{n,k}\right)$, stays inside $\mathcal{E}_{\chi^2} (x_{n,k},P_{n,k})$. Subsequently, $\mathcal{E}_{\chi^2}(x[\lambda], P[\lambda]) \subset \mathcal{E}_{\chi^2} (x_{n,k},P_{n,k})$. $\qed$
\end{proof}

Finally, we stress that as shown in proof of Lemma~\ref{lemma:rn} the distance $\hat{D}(b_1,b_2)$ between any $b_1=(x_1,P_1)\in\mathcal{R}_{n,k}$ and any $b_2=(x_2,P_2)\in\mathcal{R}_{n,k+1}$ for all $k\in[1;K_n-1]$ is less than $3\ell_n :=D_{\text{min}}$ so the connection between $b_1$ and $b_2$ will be attempted by Algorithm \ref{alg:igpp_body1}. 


\subsection{Probability of event \texorpdfstring{$E_n$}{En}\label{subsec:event_E}}
\label{sec:enent_E_n}
Let $b_i:=(x_i,P_i) \in B$ be sample belief state by Algorithm~\ref{alg:igpp_body1} at iteration $i$. Then, due to the uniform distribution the probability of event that $x$ is samples inside $\mathcal{B}(x_{n,k},\ell_n)$ is given by
\begin{align} \nonumber
    \mathbb{P}(\{x_i\in\mathcal{B}(x_{n,k},\ell_n)\})\!&=\!\frac{\text{vol}(\mathcal{B}(x_{n,k},\ell_n))}{\text{vol}(\mathcal{X}_{\text{free}})}
    \\ 
    &=\frac{\tau_d\ell^{d}_n}{\mathcal{V}_{\mathcal{X}}}
    =  \frac{\tau_d\delta_n^{d} h^d}{  \mathcal{V}_{\mathcal{X}}}=: g_1 \delta_n^{d},
    \!\nonumber   
\end{align}
where $\mathcal{V}_{\mathcal{X} }:=\text{vol}(\mathcal{X}_{\text{free}})$ and $g_1:= \frac{\tau_d  h^d}{  \mathcal{V}_{\mathcal{X}}}$.
If we define
\begin{align}
\small
    \mathcal{D}_{n, k} :=&\{P\in\mathbb{S}^d_{4\rho}:\nonumber\\
    &\left(1-\frac{3 \delta_n}{2}\right) P_{n,k} \preceq P\preceq \left(1-\frac{\delta_n}{2}\right) P_{n,k}\},\nonumber
\end{align} 
we have $\left(1-\frac{3 \delta_n}{2}\right) P_{n,k} \succeq \left(1-\frac{3}{4}\right) \underline{\sigma}(P_{n,k}) I \succeq \rho I$. Thus, Theorem~\ref{theo:cont} holds in these regions.
We have  $\textup{Tr}(P) \geq 4 (1-\frac{3 \delta_n}{2})  \rho d \geq \rho d$ for all $P \in \mathcal{D}_{n, k}$.
On the other hand, $\textup{Tr}(P)\leq (1-\frac{\delta_n}{2}) \textup{Tr}(P_{n,k}) \leq  (1-\frac{\delta_n}{2}) R \leq R $. Hence, $\mathcal{D}_{n,k} \subset \mathcal{R}_{ [(1-\frac{3 \delta_n}{2})  4\rho d, (1-\frac{\delta_n}{2}) R ]} \subset \mathcal{R}_{ [ \rho d, R ]}$.

From the definition of uniform sampling, we have 
\begin{align*}
     &\mathbb{P}(P_i\in\mathcal{D}_{n,k})
     {\geq}\frac{\text{vol}(\mathcal{D}_{n,k}\cap \mathcal{D}_{[\rho d, R]})}{\text{vol}(\mathcal{D}_{[\rho d ,R]})} = \frac{\text{vol}(\mathcal{D}_{n,k})}{\text{vol}(\mathcal{D}_{[\rho d,R]})} \\
    &=\frac{\text{vol}(\mathcal{D}_{(1-\frac{\delta_n}{2})P_{n,k}}) - \text{vol}(\mathcal{D}_{(1-\frac{3\delta_n}{2})P_{n,k}})}{\text{vol}(\mathcal{R}_{R}) - \text{vol}(\mathcal{R}_{\rho d})}\\
        &\geq\frac{V_dSb^{\frac{d(d+1)}{2}}\bigg[ (1-\frac{\delta_n}{2})^{\frac{d(d+1)}{2}}- (1-\frac{3\delta_n}{2})^{\frac{d(d+1)}{2}} \bigg]}{\left(\frac{2}{d(d+1)}\right)V_r\bigg[R^{\frac{d(d+1)}{2}}- (\rho d)^{\frac{d(d+1)}{2}} \bigg]}.
\end{align*}
where $S=S_d(1,1,1/2)$ and $b=\underset{k\in[1;d]}{\min}\;\lambda_k(P_{n,k})$.
\begin{figure}[ht!]
\centering
\includegraphics[width=1.05\columnwidth]{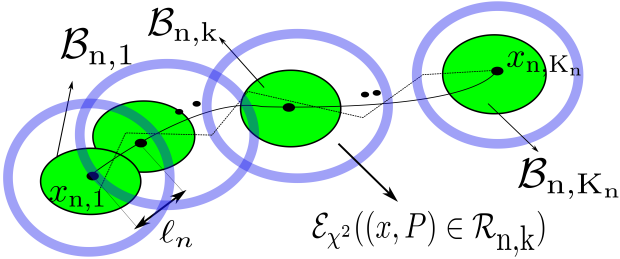}
\caption{Covering of chain $\gamma_n$ with balls of radius $\ell_n$. The event $E_{n,k}$ is the event that a sampled point $(x,P)\in\mathcal{R}_{n,k}$.}
\label{fig:covering_balls}
\end{figure}
Now, we have
\begin{align*}
    &\left(1-\frac{\delta_n}{2}\right)^{\frac{d(d+1)}{2}}-\left(1-\frac{3 \delta_n}{2}\right)^{\frac{d(d+1)}{2}}\\
    & = \bigg(\left(1-\frac{ \delta_n}{2}\right) - \left(1-\frac{3 \delta_n}{2}\right)\bigg) \\
     &\qquad \sum_{j=0}^{\frac{d(d+1)}{2}-1} \left(1-\frac{\delta_n}{2}\right)^{\frac{d(d+1)}{2}-1-j}  \left(1-\frac{3 \delta_n}{2}\right)^{j}\\
        &\geq \delta_n d \left(1-\frac{3\delta_n}{2}\right)^{\frac{d(d+1)}{2}-1}\geq   d\left(\frac{\delta_n}{2}\right)^{\frac{d(d+1)}{2}},
\end{align*}
where in the last step we use $\delta_n\leq 1/2$. 
In sum, it can be deduced
\begin{align}
    \mathbb{P}(P_i\in\mathcal{D}_{n,k}) \geq g_2 \delta_n^{\frac{d(d+1)}{2}},
\end{align}
where $g_2=V_dS_d(1,1,1/2)2^{-1}(d(d+1))V_r^{-1}(2\rho)^{\frac{d(d+1)}{2}}\bigg[R^{\frac{d(d+1)}{2}}- (\rho d)^{\frac{d(d+1)}{2}} \bigg]^{-1}$ is a constant. 

Lets define the event $E^i_{n,k}$ is as the event that the sampled belief $b_i=(x_i,P_i)$ belongs to $\mathcal{R}_{n,k}$. Then, $E_{n,k}= \cup_{i=1}^n E^i_{n,k}$ which means the event that $\mathcal{R}_{n,k}$ contains at least one sample belief $b_i \in B$. The following lemma shows that if the $\gamma$ is greater than a certain positive threshold, then the probability that event $E_n\triangleq E_{n,1}\cap E_{n,2}\dots \cap E_{n,K_n}$  occurs, i.e the event that all $R_{n,k}$'s contain at least one sampled belief, equals to one as $n$ approaches infinity.

\begin{lemma}
\normalfont If
$\gamma>\left(\frac{d^2+3d+2}{g_1 g_2 d(d+3)} \right)^\frac{2}{d(d+3)}$,
then the following holds true
\begin{align}
        \underset{n\rightarrow\infty}{\lim}\; \mathbb{P}\left(E_n \right)=1\nonumber.
\end{align}
\label{lemma:event_E_n}
\end{lemma}
\begin{proof}
The probability of event $E^c_{n,k}$ for all $k\in[1;K_n]$ is given as follows:
\begin{align}
   &\mathbb{P}\left(E^c_{n,k}\right)
   = \prod_{i=1}^{n}(1-\mathbb{P}\left(E^i_{n,k}\right))\nonumber\\
   &\leq \left(1- \mathbb{P}(P\in\mathcal{D}_{n,k}) \;     \mathbb{P}(x\in\mathcal{B}_{n,k})\right)^n\nonumber\\ \nonumber
   &\leq\left[1- g_2 \delta_n^d g_1 \delta_n^{\frac{d(d+1)}{2}} \right]^n = \left[1- g_1 g_2 \delta_n^{\frac{d(d+3)}{2}} \right]^n.
    \end{align}
Using the fact that $(1-x)\leq e^{-x}$ for $x\in(0,1)$, and substituting $\delta_n$ defined as
\begin{align}
    \delta_n=\min\;\left\{\gamma\left(\frac{\text{log}n}{n}\right)^\frac{2}{d(d+3)},\frac{1}{2}\right\},\nonumber
\end{align}
we have
\begin{align} \nonumber
 \mathbb{P}\left(E^c_{n,k}\right)&\leq \left(1-g_1 g_2
 {\gamma^{\frac{d(d+3)}{2}}\frac{\text{log}n}{n}} \right)^{n} \\ \nonumber
 &\leq \text{exp} \left( - (\log n) g_1 g_2 \gamma^{\frac{d(d+3)}{2}} \right)
 = n^{- g_1 g_2 \gamma^{\frac{d(d+3)}{2}}}.
\end{align}
Now, the event $E_n^c=\bigcup_{k=1}^n E^c_{n,k}$ is upper bounded as follows:
\begin{align}
  \mathbb{P}\left(E^c_{n}\right)& \!=\!\mathbb{P}\left(\bigcup_{k=1}^{K_n} E^c_{n,k}\right) 
 \! \leq \! \sum_{k=1}^{K_n}\mathbb{P}\left( E^c_{n,k}\right)\!\leq \! K_n n^{- g_1 g_2 \gamma^{\frac{d(d+3)}{2}}}.  \nonumber
\end{align}
where $K_n=\lfloor\frac{\ell^\star}{\ell_n}\rfloor\leq\frac{\ell^\star}{\ell_n}$.
Since $\delta_n h\leq \ell_n$, 
we have
\begin{align} \nonumber
    \mathbb{P}\left(E^c_{n}\right)
    &\leq \frac{ \ell^\star}{h \delta_n }  n^{- g_1 g_2 \gamma^{\frac{d(d+3)}{2}}}  \\ \label{eq:P_c} 
     &= \frac{ \ell^\star}{\gamma h} (\log n)^{-\frac{2}{d(d+3)}} n^{- g_1 g_2 \gamma^{\frac{d(d+3)}{2}}+ \frac{2}{d(d+3)}}. 
\end{align}
If the power of $n$ in \eqref{eq:P_c} is less than $-1$, which is equivalent to
\[
\gamma>\left(\frac{d^2+3d+2}{g_1 g_2 d(d+3)} \right)^\frac{2}{d(d+3)},
\]
we have $\sum_{n=1}^\infty  \mathbb{P}\left(E^c_{n}\right)<\infty$. Consequently, 
$\underset{n\rightarrow\infty}{\lim}\mathbb{P}\left(E^c_{n}\right)=1$,
by Borel Cantelli lemma \cite{grimmett2020probability} which completes the proof. $\qed$
\end{proof}

\subsection{Construction of chain \texorpdfstring{$\gamma'_{n,k}$}{g'}}
\label{subsec:gamma_def}
For each $n\in\mathbb{N}$, define a collision-free path $\gamma_n^{\star\star}: [0,1]\rightarrow \mathbb{R}^d\times \mathbb{S}_\rho^d$ by $\gamma_n^{\star\star}(t)=(x^\star(t), (1-\delta_n) P^\star(t))$ (See Figure~\ref{fig:1}).
It is trivial to check $ \underset{n\rightarrow\infty}{\lim} \gamma^{\star\star}_n = \gamma^\star$ and thus $\underset{n\rightarrow\infty}{\lim}c(\gamma^{\star\star}_n) = c^\star$.

We define a chain $\gamma'_n := (x_{n,k}, (1-\delta_n) P_{n,k})$, $k\in[1;K_n]$, which is the chain generated by partition $\mathcal{P}_n$ on path $\gamma_n^{\star\star}(t)$. It is easy to verify that $(x_{n,k}, (1-\delta_n) P_{n,k}) \in \mathcal{R}_{n,k}$. Thus, from Lemma~\ref{lemma:rn} one can deduce that $\gamma'_n$ is collision-free.   
Using $\{(x_{n,k}, P'_{n,k}= (1-\delta_n) P_{n,k})\}_{k=1}^{K_n}$,
we define path $\widetilde{\gamma}_n(t)=(\widetilde{x}_n(t), \widetilde{P}_n(t))$ as
\begin{align*}
\widetilde{x}_n(t)&= \frac{t_{n,k+1}-t}{t_{n,k+1}-t_{n,k}}x_{n,k}+\frac{t-t_{n,k}}{t_{n,k+1}-t_{n,k}}x_{n,k+1} \\
\widetilde{P}_n(t)&=P'_{n,k}+\|\tilde{x}_n(t)-x'_{n,k}\| W
\end{align*}
for $t_{n,k}\leq t < t_{n,k+1}, k\in[1;K_n-1]$ with $\widetilde{x}(1)=x_{n,K_n}$, $\widetilde{P}(1)=P'_{n,K_n}$.
Notice that we have
\begin{equation}
\label{eq:c_gamma'}
c^\star \leq c(\widetilde{\gamma}_n)=c(\gamma'_n)\leq c(\gamma^{\star\star}_n)
\end{equation}
since $\widetilde{\gamma}_n$ is collision-free path from initial belief state to the goal region (the first inequality) and
\begin{align}
c(\gamma'_n)=c(\gamma^{\star\star}_n;\mathcal{P}_n) \leq \sup_{\mathcal{P}}c(\gamma^{\star\star}_n;\mathcal{P})=c(\gamma^{\star\star}_n).
\label{eqn:cost_gamma_dash_equals_gamma}
\end{align}
Since $\underset{n\rightarrow\infty}{\lim} c(\gamma^{\star\star}_n)=c^\star$, \eqref{eq:c_gamma'} implies $\underset{n\rightarrow\infty}{\lim} c(\gamma'_n)=c^\star$.



\subsection{Convergence to optimal path \texorpdfstring{ $\gamma^\star$}{g}}
\label{sec:convergence_to_optimal}

In subsection \ref{sec:enent_E_n}, we showed $ \underset{n\rightarrow\infty}{\lim}\; \mathbb{P}\left(E_n \right)=1\nonumber$. In this subsection, we assume $E_n$ has occurred and show almost surely $G$ contains a path whose cost is arbitrarily close to $c^\star$ in the limit of large $n$. Let the set of paths generated by the IG-PRM* algorithm be denoted by $P_n$ and let $\gamma^p_n$ be closest to $\gamma'_n$ in terms of bounded variation i.e., $\gamma^p_n=\underset{\gamma^p\in P_n}{\argmin}\;\|\gamma^p-\gamma'_n\|_{TV}$. Then, we have the following lemma.
\begin{lemma}
$\mathbb{P}\left(\left\{\underset{n\rightarrow\infty}{\lim}\;\left\|\gamma^p_{n}-\gamma'_{n}\right\|_{TV}=0\right\}\right)=1$
\label{lemma:arbritrarily_close}
\end{lemma}
\begin{proof}
Define set $\mathcal{R}^{\beta}_{n,k}\subseteq \mathcal{R}_{n,k} $ 
\begin{align} \nonumber
    &\mathcal{R}^{\beta }_{n,k}=\Biggl\{(x,P): x\in\mathcal{B}\left(x_{n,k},\beta \ell_n\right),\\ \nonumber
    &{\left(1-\delta_n-\beta \frac{\delta_n}{2}\right)}P_{n,k}\preceq P \preceq  \left(1-\delta_n+\beta \frac{\delta_n}{2}\right) P_{n,k}\Biggl\},
\end{align}
where $0<\beta\leq1$. 
Define $I_{n,k}$ as follows:
\begin{align}
    I_{n, k}:= \begin{cases}1, & \text { if }\mathcal{R}^{\beta}_{n,k} \cap V^{\mathrm{IG-PRM}^{*}}=\emptyset, \\ 0, & \text { otherwise. }\end{cases}
\end{align}
We define $M_n=\sum_{k=1}^{K_n} I_{n,k}$, and examine the event $\{M_n\leq \alpha K_n\}$ which means the event that at least $\alpha$ fraction of all $K_n$ regions (i.e., $R_{n,k}$s regions)  do not contain any sampled belief $b_i \in B$. If 
$\mathcal{R}^{\beta}_{n,k}$ contains
a sampled belief state $(x,P)$, we have \begin{align*}
     \|x-x_{n,k}\| & \leq \beta \ell_n,  \\
    \|P-P_{n,k}\|_F &\leq  \|(1-\delta_n + \beta \frac{\delta_n}{2})  P_{n,k}- (1-\delta_n)P_{n,k}\|_F \\
    &\leq \left((1-\delta_n + \beta \frac{\delta_n}{2})- (1-\delta_n) \right) \|P_{n,k}\|_F\\
    & \leq \beta \frac{\delta_n}{2}  \|P_{n,k}\|_F \\
    & \leq \beta\frac{\delta_n}{2} \|P_{n,k}\|_F \leq \beta \frac{ R}{2h} \ell_n ,
\end{align*}
which yields $\hat{\mathcal{D}}((x,P), (x_{n,k},P_{n,k})) \leq  \beta c \ell_n$, where $c:= 1+R/2h$. For $k$'s that no sample belief is inside $\mathcal{R}^{\beta}_{n,k}$, we can assume there exists a sampled belief $(x,P)$ inside $\mathcal{R}_{n,k}$, because we have assumed $E_n$ has occurred. For such $k$s, we have $\hat{\mathcal{D}}((x,P), (x_{n,k},P_{n,k})) \leq  c \ell_n$ because $ \mathcal{R}_{n,k}= \mathcal{R}^{\beta=1}_{n,k}$.
Considering the path generated by connecting the sampled belief in consequent $\mathcal{R}^\beta_{n,k}$ or $\mathcal{R}_{n,k}$ regions, we have 
\begin{align}
\left\|\gamma^p_{n}-\gamma'_{n}\right\|_{TV}& \leq  \sum_{k=1}^{K_n}  \hat{\mathcal{D}}((x,P), (x_{n,k},P_{n,k})\nonumber\\
&\leq K_n (\alpha c \ell_n + (1-\alpha) \beta c \ell_n) \nonumber\\
&\leq c(\alpha+\beta) L,
\end{align}
where { $L=\underset{n\in\mathbb{N}}{\sup}\;\gamma'_n$}. Therefore,  
\begin{align}
&\left\{M_{n} \leq \alpha K_{n}\right\} \subseteq
\left\{\left\|\gamma^p_{n}-\gamma'_{n}\right\|_{TV} \leq c(\alpha+\beta) L \right\}\nonumber.
\end{align}
Taking the complement of both sides of above equation and using the monotonicity of probability measures,
\begin{align}
&\mathbb{P}\left(\!\left\{\left\|\gamma^p_{n}-\gamma'_{n}\right\|_{TV}\!>\! c(\alpha+\beta) L \right\}\!\right)
\!\leq\!\mathbb{P}\left(\!\left\{M_{n} \!\geq\! \alpha K_{n}\right\}\!\right).
\label{eqn:complement}
\end{align}
Since \eqref{eqn:complement} is true for any $\alpha,\beta\in(0,1)$, it remains to show that $\mathbb{P}\left(\left\{M_{n} \geq \alpha K_{n}\right\}\right) $ converges to zero. 
Let's denote $\ell^\beta_{ n,k}:=\beta \ell_n)$, $\mathcal{D}^\beta_{n,k}= \{ P\succ 0: {(1-\delta_n-\beta \frac{\delta_n}{2})}P_{n,k}\preceq P \preceq  (1-\delta_n+\beta \frac{\delta_n}{2}) P_{n,k}\}$. Then, expected value of $I_{n,m}$ can be computed as 
\begin{align}
    &\mathbb{E}[I_{n,m}]
    = \mathbb{P}(\{I_{n,m}=1\}) \nonumber\\
    &\leq \left( 1- \mathbb{P}(x\in\mathcal{B}(x,\ell^\beta_{n,k})) \times  \mathbb{P}(P\in\mathcal{D}^\beta_{n,k}) \right)^n\nonumber\\
    &\leq \left( 1 - g_1 \beta^d \delta^{d}_{n}  \times g_2 \beta \delta_n^{\frac{d(d+1)}{2}}  \right)^n\nonumber\\
    &=
    \left(1-g_1 g_2 \beta^{d+1} \delta_n^{\frac{d(d+3)}{2}}\right)^{n}\nonumber\\ \nonumber
    &\leq \text{exp} \left(-n g_1 g_2 \beta^{d+1} \delta_n^{\frac{d(d+3)}{2}} \right)\\
    &=   \text{exp} \left(-g_1 g_2 \beta^{d+1} \gamma^{\frac{d(d+3)}{2}} \frac{\log n}{n}  n\right) \nonumber \\ \nonumber
    &\leq \text{exp} \left(-\frac{-g_1g_2\beta^{d+1}(d^2+3d+2)}{d(d+3)}\log n\right)\nonumber\\
    &= n^{\frac{-\beta^{d+1}g_1g_2(d^2+3d+2)}{d(d+3)}}. \nonumber
\end{align}
Thus, $\mathbb{E}[M_n]=\sum_{m=1}^{K_n} \mathbb{E}[I_{n,m}]= K_n n^{\frac{-\beta^{d+1}g_1g_2(d^2+3d+2)}{d(d+3)}}$. By Markov’s inequality, it follows that
\begin{align*}
    \mathbb{P}\left(\left\{M_{n} \geq \alpha K_{n}\right\}\right) &\leq \frac{\mathbb{E}[M_n]}{\alpha K_n} \leq \frac{K_n n^{\frac{-\beta^{d+1}g_1g_2(d^2+3d+2)}{d(d+3)}}}{\alpha K_n}\nonumber\\
    &= \frac{ n^{\frac{-\beta^{d+1}g_1g_2(d^2+3d+2)}{d(d+3)}}}{\alpha}.  
\end{align*}
it is trivial to verify that for fixed $\alpha$, the last expression tends to $0$ as $n$ tends to $\infty$. Since this argument holds for all small $\alpha, \beta$, 
\eqref{eqn:complement} implies for all $\epsilon>0$,
$$
\mathbb{P}\left(\left\{\underset{n\rightarrow\infty}{\lim}\;\left\|\gamma^p_{n}-\gamma'_{n}\right\|_{TV}=0\right\}\right)=1. 
$$
$\qed$
\end{proof}
As shown in Subsection~\ref{subsec:gamma_def} $ \underset{n\rightarrow\infty}{\lim}c(\gamma'_n)=c^\star$.
Using this result, the fact that $\mathbb{P}\left(\left\{\underset{n\rightarrow\infty}{\lim}\;\left\|\gamma^p_{n}-\gamma'_{n}\right\|_{TV}=0\right\}\right)=1$, and the continuity of the cost function (Theorem~\ref{theo:cont}), we can conclude that
      $ \mathbb{P}(\{\underset{n\rightarrow\infty}{\lim}c(\gamma^p_n)=c^\star\})=1. $
\section{Loss-less version of IG-PRM*\label{sec:lossless_ri_prm_star}}
Note that only lossless transitions between edges of a feasible path are meaningful for physical systems \cite{pedram2021gaussian}. However, the path generated by IG-PRM* (Algorithm \ref{alg:igpp_body1}) is not guaranteed to be finitely lossless. Although there exists an analytical solution to perform lossless refinement (Algorithm \ref{alg:lossless_modify_ri_prm_star}) for the entire path using the information geometric cost defined in \eqref{eq:def_D}, an analytical solution might not exist for other information metrics such as Hellinger or Wasserstein. To address this limitation, we propose the finitely lossless version of the IG-PRM* algorithm termed Lossless IG-PRM* (Algorithm \ref{alg:lossless_ri_prm}) where the edge between two belief states is guaranteed to be finitely lossless and is devoid of any lossless refinement step. In addition, we prove the asymptotic optimality of Lossless IG-PRM*. 

Algorithm \ref{alg:lossless_ri_prm} represents the pseudocode for the Lossless IG-PRM*. The function $\operatorname{CollisionFree}(b, b_j, \chi^2)$ checks whether the edge from belief state $b$ to $b_j$ is collision-free. The function $\operatorname{Lossless}(b, b_j)$ checks whether the transition from $b$ to $b_j$ is lossless. The search algorithm (Algorithm \ref{algo:search}) is used to compute the lossless and collision-free chain $\gamma''_n$.

\begin{algorithm}
    { 
    $\left.B \leftarrow\left\{b_{\text{init}}\right\} \cup\{\text {SampleFree}_{i}\right\}_{i\in [1;n]} ; E \leftarrow \emptyset$\;
    \For{$b\in B$}{
    $\quad B_{\text{nbors}} \leftarrow \operatorname{Near}\left(B, b, D_{\text{min}}\right)$\;
    \For{$b_j\in B_{\text{nbors}}$}{
    \If{$\operatorname{CollisionFree}(b, b_j, \chi^2)\; \operatorname{and}\; \operatorname{Lossless}(b, b_j)$  }{$E \leftarrow E \cup\{(b, b_j)\}$\;}
    }
    }
    \Return $G=(B, E)$
    }
\caption{Lossless IG-PRM*}
\label{alg:lossless_ri_prm}
\end{algorithm}

Towards the aim of proving asymptotic optimality, we define modified $\ell^L_n,\;\mathcal{R}^L_{n,k},\;\mathcal{D}^L_{n,k},\;\mathcal{R}^{L,s}_{n,k},\;\mathcal{R}^{L,\beta s}_{n,k}$ and $\delta^L_n$ as follows\footnote{superscript $L$ is used to indicate regions/parameters defined for Lossless IG-PRM* algorithm}:
\begin{align}
     & \ell^L_n:=\min\left\{  \delta^L_n\left(1-\frac{3}{4}\delta^L_n\right)\frac{\rho}{3\bar{\sigma}(W)},{\frac{\delta^L_n}{9}\rho},{\frac{\delta^L_n}{3}\chi^{\frac{1}{4}}{\rho}^{\frac{1}{8}}},q_n\right\},\nonumber\\
     &\quad\geq \delta^L_n h^L \nonumber\\
             & \mathcal{B}^s(x_{n,k},\ell^L_n):=\mathcal{B}(x_{n,k},\frac{(1-(\theta^1_{n,k})^2)}{2}(\ell^L_n)^4)\nonumber\\
        &\mathcal{R}^L_{n,k}=\{(x,P): x\in \mathcal{B}(x_{n,k},(\ell^L_n)^4), P\preceq {\left(\theta^0_{n,k}\right)^2}P''_{n,k} \}\nonumber\\
           &\mathcal{R}^{L,s}_{n,k}=\{(x,P):x\in \mathcal{B}^s(x_{n,k},\ell^L_n)\},\nonumber\\
       &\mathcal{D}^L_{n, k}:=\{P\succeq0: {(\theta^1_{n,k})^2} P''_{n,k} \preceq P\preceq {(\theta^2_{n,k})^2} P''_{n,k}\},\nonumber\\
&\;\quad\quad{(\theta^1_{n,k})^2}P''_{n,k}\preceq P\preceq {(\theta^2_{n,k})^2} P''_{n,k},\} \nonumber\\
            &\mathcal{R}^{L,\beta s}_{n,k}=\Biggl\{(x,P): x\in\mathcal{B}\left(x_{n,k},\frac{\beta(1-(\theta^1_{n,k})^2)}{2}(\ell^L_n)^4\right),\nonumber\\ 
    &{(\beta\theta^1_{n,k}+(1-\beta)\theta^2_{n,k})^2}P''_{n,k}\preceq P \preceq {(\theta^2_{n,k})^2}P''_{n,k}\Biggl\}\nonumber\\
        &\delta^L_n=\min\;\left\{\gamma\left(\frac{\text{log}n}{n}\right)^\frac{1}{d(2d+8)},1\right\}\nonumber
\end{align}
where $q_n=\min\left\{\frac{\delta_n\sqrt{ \chi^2 \rho} }{16},\frac{\delta_n\rho }{18\bar{\sigma}(W)}\right\},\;\theta^1_{n,k}=\theta^2_n-(k+1)\Delta$, $\theta^2_{n,k}=\theta^2_n-k\Delta$, and $\Delta=(\theta^2_n-\theta^1_n)/K^L_n$. Here, $\theta^2_{n}:=\sqrt{{1}-\frac{\ell^2_n}{\chi^{\frac{1}{2}}\rho^{\frac{1}{4}}}}$, and $\theta^1_{n}:=\sqrt{1-(\delta^L_n)^2+\frac{\ell^2_n}{\chi^{\frac{1}{2}}\rho^{\frac{1}{4}}}}$, $h^L:=\min\left\{\frac{\rho}{12\bar{\sigma}(W)},\frac{\rho}{9}, \frac{\rho^{\frac{1}{4}}}{3}\right\}$ and $\beta\in(0,1)$ where $\theta^0_{n,k}$ is defined as follows
\begin{align}
   & \theta_{n,k}^{0}= 1- \frac{(\ell^L_n)^4}{\sqrt{\chi^2\underline{\sigma}(P''_{n,k})}}\nonumber   
\end{align}

The outline of the asymptotic optimality proof is as follows.
The quantities $\gamma^\star$, $\gamma^\star(t)$, $\gamma_n'$ and $\mathcal{P}_n$ are same as that described in Section \ref{subsection:outline_of_proof}. However, for Lossless IG-PRM*, in addition to proving that the entire path is collision-free, we also need to show that the path is finitely lossless. This necessitates us to modify some constants and regions compared to that mentioned in Section \ref{subsection:outline_of_proof} and the mathematical reason behind it would be clear in subsequent sections.
As $\gamma_n'$ might not be finitely lossless, we first consider the lossless refinement of $\gamma_n'$ to obtain a finitely lossless and collision-free chain $\gamma_n''\triangleq (x_{n,k},P''_{n,k}),\;k\in[1;K^L_n]$ and show the cost of $\gamma''$ converges to the cost of the optimal path $c^\star=c(\gamma^\star)$ in the limit of $n$ tending to infinity.
\begin{figure*}[ht]
        \centering
 \captionsetup[subfigure]{justification=centering}
 \centering
 \begin{subfigure}{0.24\textwidth}
{\includegraphics[width=4.5cm]{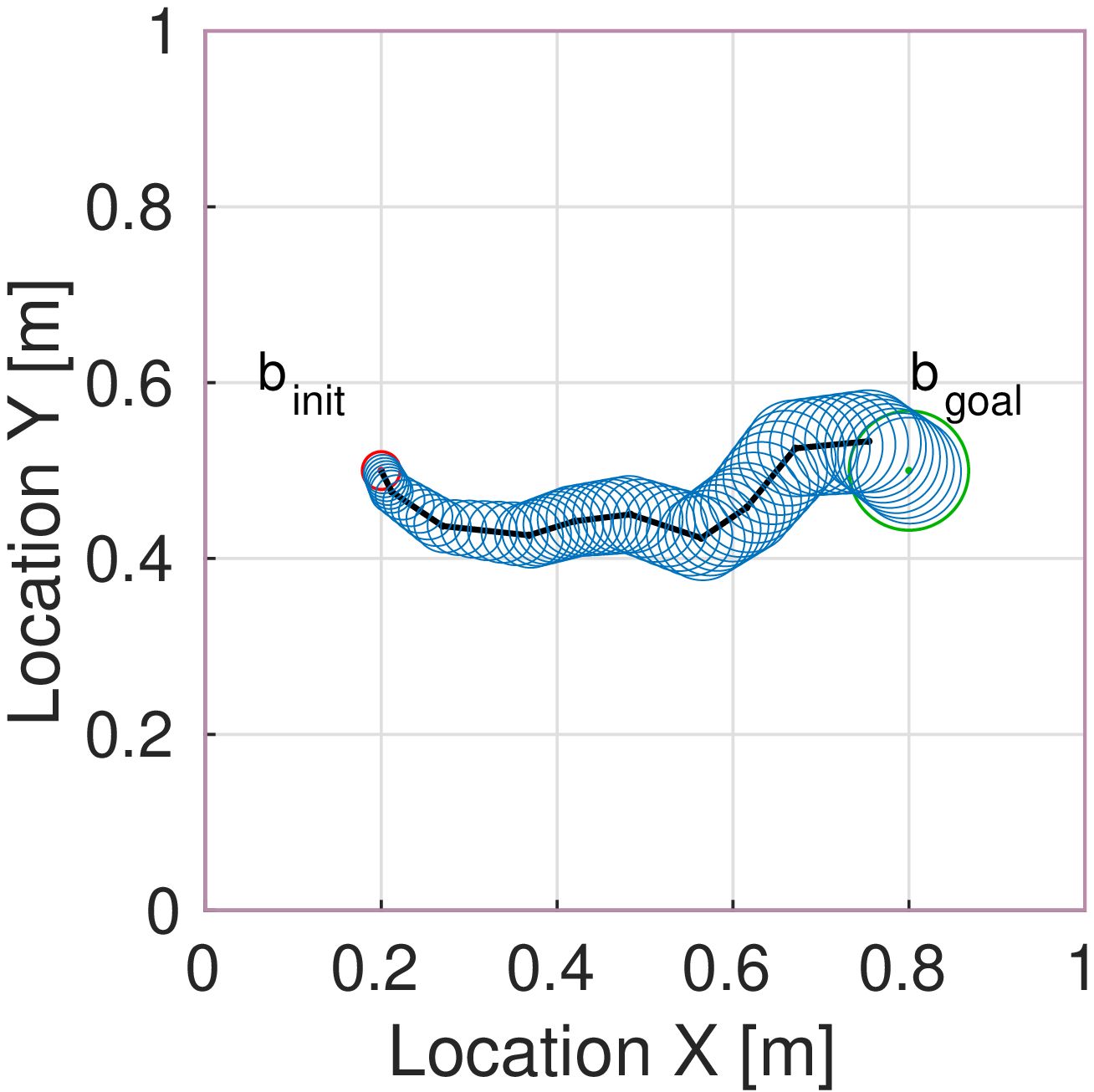}}
\caption{$c(\gamma_n)=1.12$}
 \end{subfigure}
 \begin{subfigure}{0.24\textwidth}
 \includegraphics[width=4.5cm]{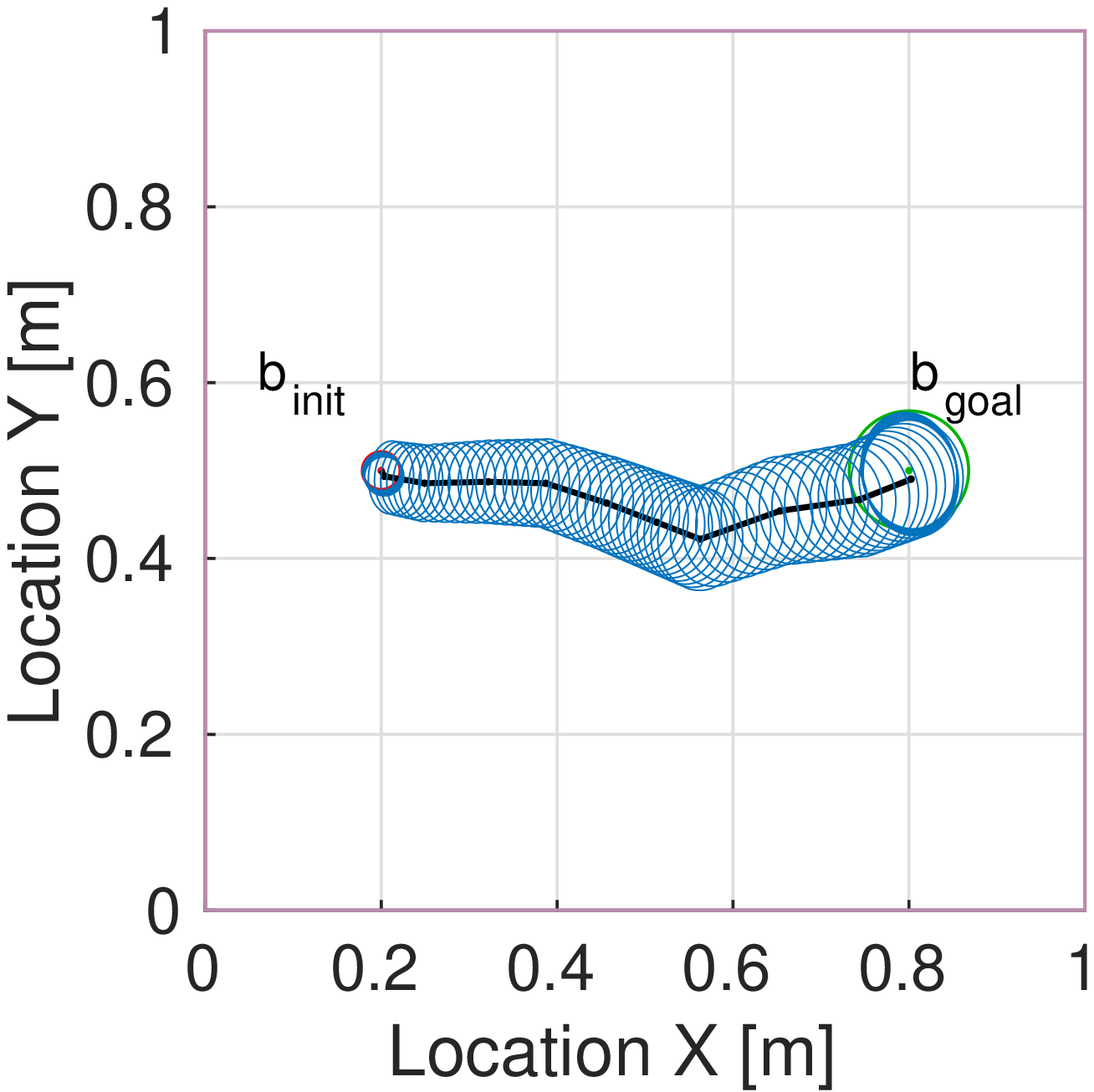}
 \caption{$c(\gamma_n)=0.90$}
 \end{subfigure}
  \begin{subfigure}{0.24\textwidth}
{\includegraphics[width=4.5cm]{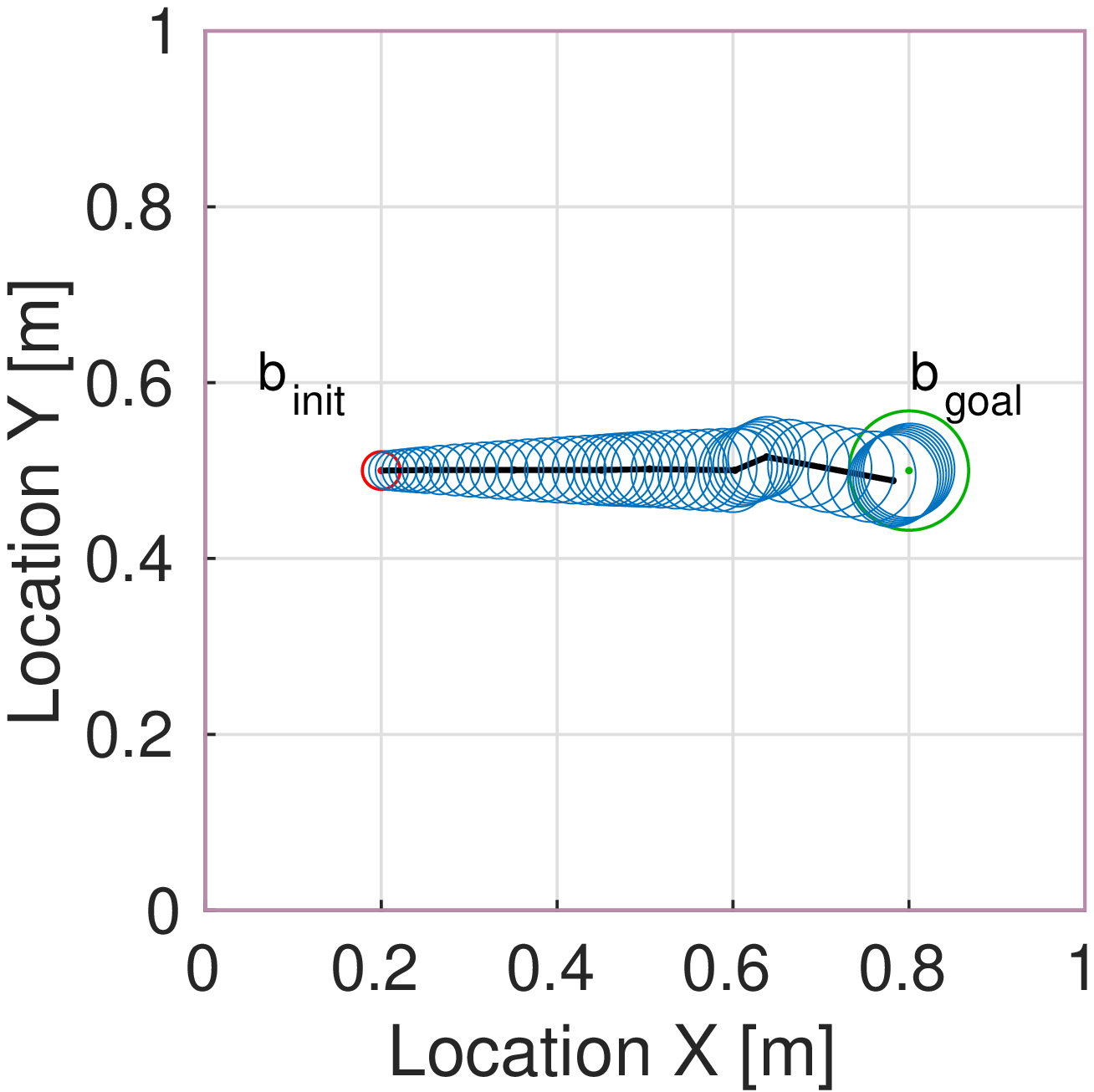}}
\caption{$c(\gamma_n)=0.61$}
 \end{subfigure}
 \begin{subfigure}{0.24\textwidth}
 \includegraphics[width=4.5cm]{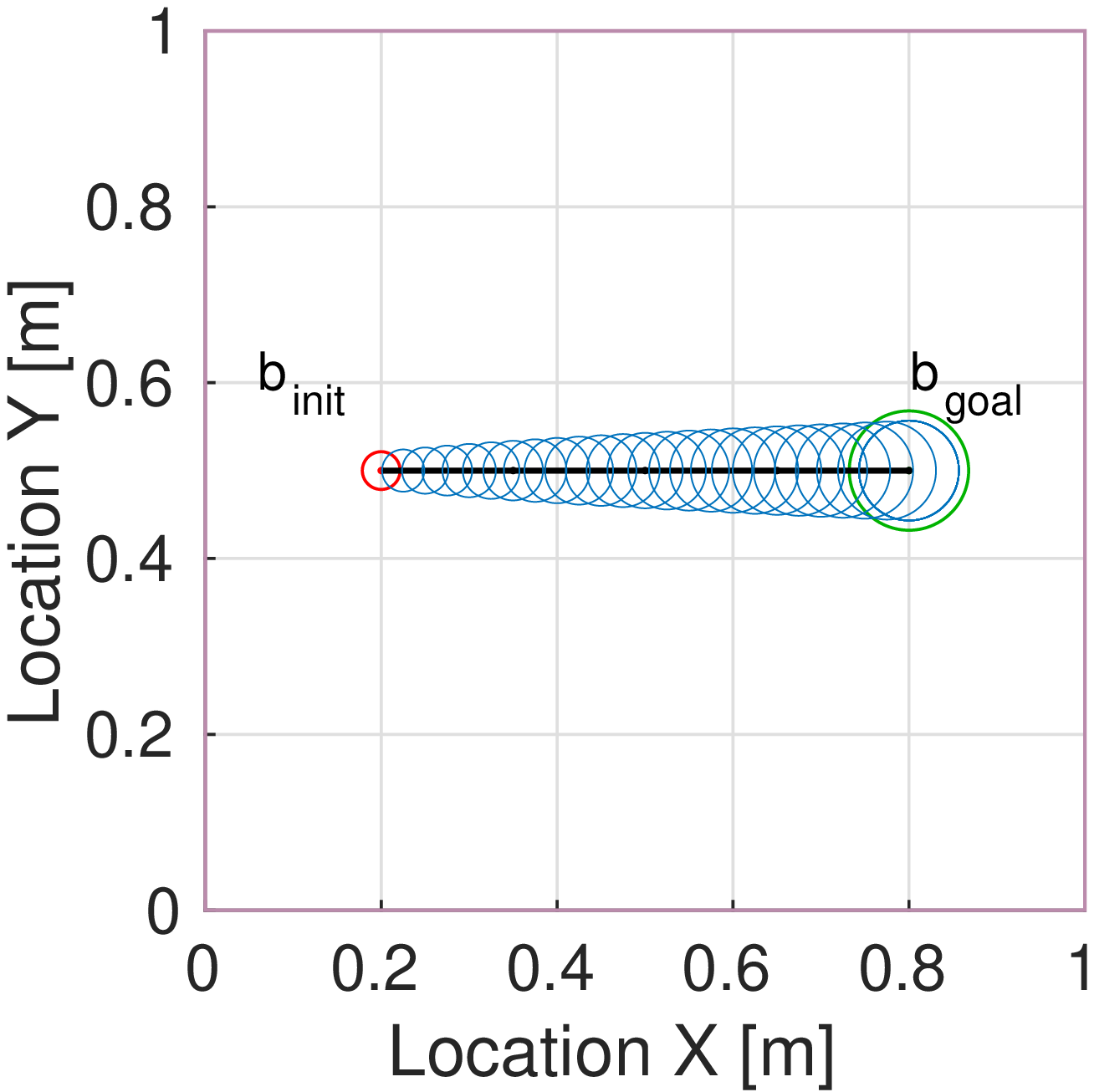}
\caption{$c^\star=0.60$}
\end{subfigure}
 \caption{Figs. (a), (b) and (c) are the paths computed by IG-PRM* after 2000, 8000, and 15000 iterations respectively. (d) The optimal path generated in the absence of obstacles from initial belief state $b_{\text{init}}$ to final belief state $b_{\text{goal}}$ using the move and sense strategy presented in \cite{pedram2021gaussian}.}
\label{fig:evolution_of_path}
\vspace{-2ex}
\end{figure*}
\begin{figure}[ht!]
\centering
\includegraphics[width=0.99\columnwidth]{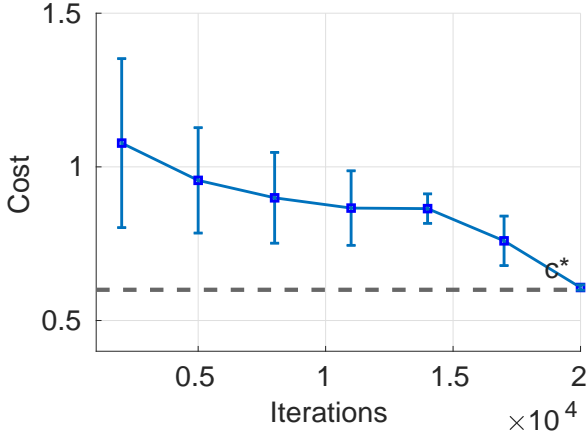}
\caption{The plot shows the convergence of the cost computed by IG-PRM* algorithm to the optimal cost $c^\star=0.60$ as the number of samples $n$ increases.}
\label{fig:conv_no_obs}
\end{figure}
To draw the connection between $\gamma''_n$ and the cost of paths returned by the Lossless IG-PRM* algorithm (Algorithm \ref{alg:lossless_ri_prm}), the regions $\mathcal{R}^L_{n,k}$ are constructed in such a way that the transition between consecutive regions i.e., between any belief state  $b_k \in \mathcal{R}^L_{n,k}$ and any belief state $b_{k+1} \in \mathcal{R}^L_{n,k+1}$ for all $k\in[1;K^L_n-1]$ is collision-free. However, as the edge connecting $b_k \in \mathcal{R}^L_{n,k}$ to any belief state $b_{k+1} \in \mathcal{R}^L_{n,k+1}$ might not be lossless, we define a region $\mathcal{R}^{L,s}_{n,k}$ which is a subset of $\mathcal{R}^L_{n,k}$ and show the transition between any $b^s_k \in \mathcal{R}^{L,s}_{n,k}$ to any $b^s_{k+1} \in \mathcal{R}^{L,s}_{n,k+1}$ is both lossless and collision-free (Lemma \ref{lemma:transition_is_lossless}). Furthermore, we show that the distance between $b_k$ and $b_{k+1}$ is smaller that $D_{\textup{min}}$ (Lemma \ref{lemma:distance_smaller_than_ed_min_lossless}). However, it is still not clear whether a connection between $b_k^s$ and $b_{k+1}^s$ will be established by Algorithm~\ref{alg:lossless_ri_prm}, if $b^s_k \in B$ and $b^s_{k+1} \in B$.
To prove this mathematically, we define event $E_{n} \triangleq E_{n,1}\cap E_{n,2}\dots \cap E_{n,K^L_n}$ as the event that a belief state is sampled inside all $\mathcal{R}^{L,s}_{n,k}$ regions, and show that the event $E_n$ occurs with probability one as $n$ tends to infinity (Lemma \ref{lemma:event_E_n_for_lossless}). 
We then show there exists a path on the graph generated by the Lossless IG-PRM$^\star$ algorithm that gets arbitrarily close to $\gamma''_n$ as $n$ tends to infinity (Lemma \ref{lemma:arbritrarily_close_for_lossless}). Finally, we leverage the continuity of path cost function to show that the cost of that path returned by Lossless IG-PRM* gets arbitrarily close to $c^\star$ (Lemma \ref{lemma:cont_final_lossless}).
\begin{lemma}
\label{lemma:chain_col_free_lossless}
\normalfont The chain $\gamma'_n = (x_{n,k}, P'_{n,k})\triangleq (x^\star(t_{n,k}), (1-\delta^L_n)^2 P^\star(t_{n,k}))\triangleq (x_{n,k}, P'_{n,k})\triangleq (x_{n,k}, (1-\delta^L_n)^2 P_{n,k}) $, $k\in[1;K_n]$ is collision-free. 
\end{lemma}
\begin{proof}
    Please see Appendix \ref{appendix:chain_col_free_lossless}
\end{proof}
\begin{lemma}
 \normalfont In any transition from $(x_1, P_1) \in \mathcal{R}^L_{n,k}$  $(x_2, P_2) \in \mathcal{R}^L_{n,k+1}$, both initial confidence ellipse and final confidence ellipse are contained inside $\mathcal{E}_{\chi^2} (x_{n,k},P_{n,k})$. More precisely,
\begin{align}
\label{eq:init_in}
&\mathcal{E}_{\chi^2}(x_1, P_1) \subseteq \mathcal{E}_{\chi^2} (x_{n,k},P_{n,k}),
\\ \label{eq:final_in_lossless}
&\mathcal{E}_{\chi^2}(x_2, P_1+||x_2-x_1||W) \subseteq \mathcal{E}_{\chi^2} (x_{n,k},P_{n,k}),
\end{align}
which proves the transition $(x_1,P_1)\rightarrow (x_2,P_2)$ fully resides in $\mathcal{E} _{\chi^2}(x_{n,k},P_{n,k})$, and thus it is collision-free. 
\label{lemma:collision_free_lossless}
\end{lemma}
\begin{proof}
    Please see Appendix \ref{appendix:collision_free_lossless}.
\end{proof}

\begin{lemma}
\normalfont The transition between any state $(x_1,P_1)\in\mathcal{R}^{L,s}_{n,k}$ to $(x_2,P_2)\in\mathcal{R}^{L,s}_{n,k+1}$ for all $k\in[1;K_n-1]$ is finitely lossless and collision-free.
\label{lemma:transition_is_lossless}
\end{lemma}
\begin{proof}
    Please see Appendix \ref{appendix:transition_is_lossless}.
\end{proof}

\begin{lemma}
\normalfont The distance $\hat{D}(b_1,b_2)$ between any $b_1=(x_1,P_1)\in\mathcal{R}^{L,s}_{n,k}$ and any $b_2=(x_2,P_2)\in\mathcal{R}^{L,s}_{n,k+1}$ for all $k\in[1;K_n-1]$ is less than $D_{\text{min}}$.
\label{lemma:distance_smaller_than_ed_min_lossless}
\end{lemma}
\begin{proof}
    Please see Appendix \ref{appendix:distance_smaller_lossless}.
\end{proof}

\begin{lemma}
\normalfont If
$\gamma>\left(\frac{d(2d+8)+1}{g^L_1 g^L_2 d(2d+8)} \right)^\frac{1}{d(2d+8)}$, 
then $\underset{n\rightarrow\infty}{\lim}\; \mathbb{P}\left(E_n \right)=1$.
\label{lemma:event_E_n_for_lossless}
\end{lemma}
\begin{proof}
    Please see Appendix \ref{appedix:event_en_lossless}.
\end{proof}

\begin{lemma}
$\mathbb{P}\left(\left\{\underset{n\rightarrow\infty}{\lim}\;\left\|\gamma^p_{n}-\gamma''_{n}\right\|_{TV}=0\right\}\right)=1$
\label{lemma:arbritrarily_close_for_lossless}
\end{lemma}
\begin{proof}
    Please see Appendix \ref{appendix:arbitr_close_lossless}.
\end{proof}
\begin{lemma}
   $ \mathbb{P}(\{\underset{n\rightarrow\infty}{\lim} c(\gamma^p_n)=c^\star\})=1$
\label{lemma:cont_final_lossless}
\end{lemma}
\begin{proof}
    Please see Appendix \ref{appendix:cont_final_lossless}.
\end{proof}

\section{Numerical Experiments}
\label{sec:experiments}
We consider a scenario with no obstacles and verify that Algorithm \ref{alg:igpp_body1} converges to the analytically computed cost as the number of samples increases. 
\subsection{Obstacle free space}
In obstacle-free space, we compute the optimal solution to \eqref{eq:def_D} by using Theorem 1 in \cite{pedram2021gaussian}. We then use the IG-PRM* algorithm (Algorithm \ref{alg:igpp_body1}) to plot the cost returned as the number of samples $n$ increases. The obstacle-free space considered is of dimension $1\mathrm{m}\times1\mathrm{m}$ with initial belief state $b_{\text{init}}:=(x_{\text{init}},P_{\text{init}})$ and goal belief state $b_{\text{goal}}:=(x_{\text{goal}},P_{\text{goal}})$ where $x_{\text{init}},\;P_{\text{init}},\;x_{\text{goal}}$ and $P_{\text{goal}}$ are defined as 
\begin{align}
    &x_{\text{init}}=[0.2,\;0.5]^\mathrm{T},\quad x_{\text{goal}}=[0.8,\;0.5]^\mathrm{T}\nonumber\\
    &P_{\text{init}}=10^{-4}I,\quad P_{\text{goal}}=10^{-3}I\nonumber
\end{align}
We perform a total of 70 experiments where the IG-PRM* is run for iterations $i=2000+3000m$ for $m\in[0;6]$ and for each $m$, 10 experiments were performed. The average cost is calculated and the results are depicted as an error plot as shown in Fig. \ref{fig:conv_no_obs}. From Fig. \ref{fig:conv_no_obs}, it can be observed that the cost returned by IG-PRM* converges to the optimal cost $c^\star$ as the number of samples increases. Fig. \ref{fig:evolution_of_path} shows the evolution of the path as the number of samples increases. 


\section{Conclusion and Future Work}
 In this paper, we proposed a sampling based motion planning algorithm in Gaussian belief space termed IG-PRM* that minimizes the information geometric cost and proved that the algorithm converges to the optimal solution as the number of samples tends to infinity. We then proposed a variant of IG-PRM* termed Lossless IG-PRM* which does not require lossless modification and prove that the algorithm is asymptotically optimal. In the future, we plan to reduce the computational burden for the IG-PRM* and explore methods to increase the rate of converge of the cost computed by IG-PRM* 
\label{sec:conclusion}

%





\begin{acks}
This work was supported in part by the Lockheed Martin
Corporation under Grant MRA16-005-RPP009, in part by the Air Force Office
of Scientific Research under Grant FA9550-20-1-0101,
\end{acks}

\bibliographystyle{SageH}
\bibliography{main.bib}

\begin{thebibliography}{53}
\providecommand{\natexlab}[1]{#1}
\providecommand{\url}[1]{\texttt{#1}}
\providecommand{\urlprefix}{URL }
\expandafter\ifx\csname urlstyle\endcsname\relax
  \providecommand{\doi}[1]{DOI:\discretionary{}{}{}#1}\else
  \providecommand{\doi}{DOI:\discretionary{}{}{}\begingroup
  \urlstyle{rm}\Url}\fi

\bibitem[{Astrom(1965)}]{astrom1965optimal_pomdp_1}
Astrom KJ (1965) Optimal control of markov decision processes with incomplete
  state estimation.
\newblock \emph{J. Math. Anal. Applic.} 10: 174--205.

\bibitem[{Blackmore et~al.(2006)Blackmore, Li and
  Williams}]{blackmore2006probabilistic_related_to_cc_1}
Blackmore L, Li H and Williams B (2006) A probabilistic approach to optimal
  robust path planning with obstacles.
\newblock In: \emph{2006 American Control Conference}. IEEE, pp. 7--pp.

\bibitem[{Blackmore et~al.(2010)Blackmore, Ono, Bektassov and
  Williams}]{blackmore2010probabilistic_cc_extend_1}
Blackmore L, Ono M, Bektassov A and Williams BC (2010) A probabilistic
  particle-control approximation of chance-constrained stochastic predictive
  control.
\newblock \emph{IEEE transactions on Robotics} 26(3): 502--517.

\bibitem[{Blackmore et~al.(2011)Blackmore, Ono and
  Williams}]{blackmore2011chance_related_to_cc_2}
Blackmore L, Ono M and Williams BC (2011) Chance-constrained optimal path
  planning with obstacles.
\newblock \emph{IEEE Transactions on Robotics} 27(6): 1080--1094.

\bibitem[{Bry and Roy(2011)}]{bry2011rapidly_bry_and_roy}
Bry A and Roy N (2011) Rapidly-exploring random belief trees for motion
  planning under uncertainty.
\newblock In: \emph{2011 IEEE international conference on robotics and
  automation}. IEEE, pp. 723--730.

\bibitem[{Censi et~al.(2008)Censi, Calisi, De~Luca and
  Oriolo}]{censi2008bayesian_censi}
Censi A, Calisi D, De~Luca A and Oriolo G (2008) A bayesian framework for
  optimal motion planning with uncertainty.
\newblock In: \emph{2008 IEEE International Conference on Robotics and
  Automation}. IEEE, pp. 1798--1805.

\bibitem[{Chikuse(2003)}]{chikuse2003statistics}
Chikuse Y (2003) \emph{Statistics on special manifolds}, volume 174.
\newblock Springer Science \& Business Media.

\bibitem[{Choset et~al.(2005)Choset, Lynch, Hutchinson, Kantor and
  Burgard}]{choset2005principles}
Choset H, Lynch KM, Hutchinson S, Kantor GA and Burgard W (2005)
  \emph{Principles of robot motion: theory, algorithms, and implementations}.
\newblock MIT press.

\bibitem[{Deits and Tedrake(2015)}]{Deits}
Deits R and Tedrake R (2015) Efficient mixed-integer planning for uavs in
  cluttered environments.
\newblock In: \emph{International Conference on Robot. and Autom.} IEEE, pp.
  42--49.

\bibitem[{Dijkstra et~al.(1959)}]{dijkstra1959note}
Dijkstra EW et~al. (1959) A note on two problems in connexion with graphs.
\newblock \emph{Numerische mathematik} 1(1): 269--271.

\bibitem[{Ding et~al.(2019{\natexlab{a}})Ding, Gao, Wang and
  Shen}]{ding2019efficient_hkust_search}
Ding W, Gao W, Wang K and Shen S (2019{\natexlab{a}}) An efficient
  b-spline-based kinodynamic replanning framework for quadrotors.
\newblock \emph{IEEE Transactions on Robot.} 35(6): 1287--1306.

\bibitem[{Ding et~al.(2019{\natexlab{b}})Ding, Zhang, Chen and
  Shen}]{ding2019safe_hkust_optimization}
Ding W, Zhang L, Chen J and Shen S (2019{\natexlab{b}}) Safe trajectory
  generation for complex urban environments using spatio-temporal semantic
  corridor.
\newblock \emph{IEEE Robot. and Automation Letters} 4(3): 2997--3004.

\bibitem[{Folsom et~al.(2021)Folsom, Ono, Otsu and
  Park}]{folsom2021scalable_inf4}
Folsom L, Ono M, Otsu K and Park H (2021) Scalable information-theoretic path
  planning for a rover-helicopter team in uncertain environments.
\newblock \emph{International Journal of Advanced Robotic Systems} 18(2):
  1729881421999587.

\bibitem[{Gao and Shen(2016)}]{hkust}
Gao F and Shen S (2016) Online quadrotor trajectory generation and autonomous
  navigation on point clouds.
\newblock In: \emph{International Symp. on Safety, Security, and Rescue Robot.
  (SSRR)}. IEEE, pp. 139--146.

\bibitem[{Grafakos and Morpurgo(1999)}]{grafakos1999selberg}
Grafakos L and Morpurgo C (1999) A selberg integral formula and applications.
\newblock \emph{Pacific Journal of Mathematics} 191(1): 85--94.

\bibitem[{Grimmett and Stirzaker(2020)}]{grimmett2020probability}
Grimmett G and Stirzaker D (2020) \emph{Probability and random processes}.
\newblock Oxford university press.

\bibitem[{Hart et~al.(1968)Hart, Nilsson and Raphael}]{hart1968formal_astar}
Hart PE, Nilsson NJ and Raphael B (1968) A formal basis for the heuristic
  determination of minimum cost paths.
\newblock \emph{IEEE transactions on Systems Science and Cybernetics} 4(2):
  100--107.

\bibitem[{Indelman et~al.(2015)Indelman, Carlone and
  Dellaert}]{indelman2015planning_inf1}
Indelman V, Carlone L and Dellaert F (2015) Planning in the continuous domain:
  A generalized belief space approach for autonomous navigation in unknown
  environments.
\newblock \emph{The International Journal of Robotics Research} 34(7):
  849--882.

\bibitem[{Janson et~al.(2015)Janson, Schmerling, Clark and
  Pavone}]{janson2015fast_fmt}
Janson L, Schmerling E, Clark A and Pavone M (2015) Fast marching tree: A fast
  marching sampling-based method for optimal motion planning in many
  dimensions.
\newblock \emph{The International journal of robotics research} 34(7):
  883--921.

\bibitem[{Kaelbling et~al.(1998)Kaelbling, Littman and
  Cassandra}]{kaelbling1998planning_pomdp_3}
Kaelbling LP, Littman ML and Cassandra AR (1998) Planning and acting in
  partially observable stochastic domains.
\newblock \emph{Artificial intelligence} 101(1-2): 99--134.

\bibitem[{Karaman and Frazzoli(2010)}]{karaman2010incremental}
Karaman S and Frazzoli E (2010) Incremental sampling-based algorithms for
  optimal motion planning.
\newblock \emph{arXiv preprint arXiv:1005.0416} .

\bibitem[{Karaman and Frazzoli(2011)}]{karaman2011sampling}
Karaman S and Frazzoli E (2011) Sampling-based algorithms for optimal motion
  planning.
\newblock \emph{The international journal of robotics research} 30(7):
  846--894.

\bibitem[{Kavraki et~al.(1996)Kavraki, Svestka, Latombe and
  Overmars}]{kavraki1996probabilistic_prm}
Kavraki LE, Svestka P, Latombe JC and Overmars MH (1996) Probabilistic roadmaps
  for path planning in high-dimensional configuration spaces.
\newblock \emph{IEEE transactions on Robotics and Automation} 12(4): 566--580.

\bibitem[{Khan et~al.(2020)Khan, Li, Kadry and
  Nam}]{khan2020control_rrt_manipulator}
Khan AT, Li S, Kadry S and Nam Y (2020) Control framework for trajectory
  planning of soft manipulator using optimized rrt algorithm.
\newblock \emph{IEEE Access} 8: 171730--171743.

\bibitem[{Kopitkov and Indelman(2017)}]{kopitkov2017no_inf2}
Kopitkov D and Indelman V (2017) No belief propagation required: Belief space
  planning in high-dimensional state spaces via factor graphs, the matrix
  determinant lemma, and re-use of calculation.
\newblock \emph{The International Journal of Robotics Research} 36(10):
  1088--1130.

\bibitem[{Kurniawati et~al.(2012)Kurniawati, Bandyopadhyay and
  Patrikalakis}]{kurniawati2012global}
Kurniawati H, Bandyopadhyay T and Patrikalakis NM (2012) Global motion planning
  under uncertain motion, sensing, and environment map.
\newblock \emph{Autonomous Robots} 33: 255--272.

\bibitem[{Kushleyev et~al.(2013)Kushleyev, Mellinger, Powers and
  Kumar}]{Kushleyev}
Kushleyev A, Mellinger D, Powers C and Kumar V (2013) Towards a swarm of agile
  micro quadrotors.
\newblock \emph{Autonomous Robots} 35(4): 287--300.

\bibitem[{LaValle and Kuffner(2001)}]{lavalle2001rapidly_rrt}
LaValle SM and Kuffner JJ (2001) Rapidly-exploring random trees: Progress and
  prospects: Steven m. lavalle, iowa state university, a james j. kuffner, jr.,
  university of tokyo, tokyo, japan.
\newblock \emph{Algorithmic and Computational Robotics} : 303--307.

\bibitem[{Lee et~al.(2016)Lee, Kim and Kim}]{lee2016planning_rrt_aerial}
Lee H, Kim H and Kim HJ (2016) Planning and control for collision-free
  cooperative aerial transportation.
\newblock \emph{IEEE Transactions on Automation Science and Engineering} 15(1):
  189--201.

\bibitem[{Levine et~al.(2013)Levine, Luders and
  How}]{levine2013information_inf3}
Levine D, Luders B and How JP (2013) Information-theoretic motion planning for
  constrained sensor networks.
\newblock \emph{Journal of Aerospace Information Systems} 10(10): 476--496.

\bibitem[{Liu et~al.(2017)Liu, Watterson, Mohta, Sun, Bhattacharya, Taylor and
  Kumar}]{upenn}
Liu S, Watterson M, Mohta K, Sun K, Bhattacharya S, Taylor CJ and Kumar V
  (2017) Planning dynamically feasible trajectories for quadrotors using safe
  flight corridors in 3-d complex environments.
\newblock \emph{IEEE Robot. and Autom. Letters} 2(3): 1688--1695.

\bibitem[{Luders et~al.(2010)Luders, Kothari and
  How}]{luders2010chance_unknown_1}
Luders B, Kothari M and How J (2010) Chance constrained rrt for probabilistic
  robustness to environmental uncertainty.
\newblock In: \emph{AIAA guidance, navigation, and control conference}. p.
  8160.

\bibitem[{MacAllister et~al.(2013)MacAllister, Butzke, Kushleyev, Pandey and
  Likhachev}]{MacAllister}
MacAllister B, Butzke J, Kushleyev A, Pandey H and Likhachev M (2013) Path
  planning for non-circular micro aerial vehicles in constrained environments.
\newblock In: \emph{International Conference on Robot. and Autom.} IEEE, pp.
  3933--3940.

\bibitem[{Mathai(1997)}]{mathai1997jacobians}
Mathai AM (1997) \emph{Jacobians of matrix transformation and functions of
  matrix arguments}.
\newblock World Scientific Publishing Company.

\bibitem[{Mellinger and Kumar(2011)}]{mellinger}
Mellinger D and Kumar V (2011) Minimum snap trajectory generation and control
  for quadrotors.
\newblock In: \emph{International Conference on Robot. and Autom.} IEEE, pp.
  2520--2525.

\bibitem[{Mercy et~al.(2017)Mercy, Van~Parys and
  Pipeleers}]{mercy2017spline_tcst}
Mercy T, Van~Parys R and Pipeleers G (2017) Spline-based motion planning for
  autonomous guided vehicles in a dynamic environment.
\newblock \emph{IEEE Transactions on Control Systems Technology} 26(6):
  2182--2189.

\bibitem[{Mittelbach et~al.(2012)Mittelbach, Matthiesen and
  Jorswieck}]{mittelbach2012sampling}
Mittelbach M, Matthiesen B and Jorswieck EA (2012) Sampling uniformly from the
  set of positive definite matrices with trace constraint.
\newblock \emph{IEEE transactions on signal processing} 60(5): 2167--2179.

\bibitem[{Muirhead(2009)}]{muirhead2009aspects}
Muirhead RJ (2009) \emph{Aspects of multivariate statistical theory}.
\newblock John Wiley \& Sons.

\bibitem[{Ono et~al.(2015)Ono, Pavone, Kuwata and
  Balaram}]{ono2015chance_cc_extend_3}
Ono M, Pavone M, Kuwata Y and Balaram J (2015) Chance-constrained dynamic
  programming with application to risk-aware robotic space exploration.
\newblock \emph{Autonomous Robots} 39(4): 555--571.

\bibitem[{Pairet et~al.(2021)Pairet, Hern{\'a}ndez, Carreras, Petillot and
  Lahijanian}]{pairet2021online_unknown_2}
Pairet {\`E}, Hern{\'a}ndez JD, Carreras M, Petillot Y and Lahijanian M (2021)
  Online mapping and motion planning under uncertainty for safe navigation in
  unknown environments.
\newblock \emph{IEEE Transactions on Automation Science and Engineering} .

\bibitem[{Pedram et~al.(2022)Pedram, Funada and Tanaka}]{pedram2021gaussian}
Pedram AR, Funada R and Tanaka T (2022) Gaussian belief space path planning for
  minimum sensing navigation.
\newblock \emph{IEEE Transactions on Robotics} : 1--20.

\bibitem[{Pedram et~al.(2021)Pedram, Stefan, Funada and
  Tanaka}]{pedram2021rationally}
Pedram AR, Stefan J, Funada R and Tanaka T (2021) Rationally inattentive
  path-planning via {RRT}*.
\newblock In: \emph{2021 American Control Conference (ACC)}. IEEE, pp.
  3440--3446.

\bibitem[{Plaku et~al.(2010)Plaku, Kavraki and
  Vardi}]{plaku2010motion_unknown_3}
Plaku E, Kavraki LE and Vardi MY (2010) Motion planning with dynamics by a
  synergistic combination of layers of planning.
\newblock \emph{IEEE Transactions on Robotics} 26(3): 469--482.

\bibitem[{Platt~Jr et~al.(2010)Platt~Jr, Tedrake, Kaelbling and
  Lozano-Perez}]{platt2010belief_platt_2010}
Platt~Jr R, Tedrake R, Kaelbling L and Lozano-Perez T (2010) Belief space
  planning assuming maximum likelihood observations .

\bibitem[{Prentice and Roy(2009)}]{prentice2009belief_brm}
Prentice S and Roy N (2009) The belief roadmap: Efficient planning in belief
  space by factoring the covariance.
\newblock \emph{The International Journal of Robotics Research} 28(11-12):
  1448--1465.

\bibitem[{Roy et~al.(1999)Roy, Burgard, Fox and Thrun}]{roy1999coastal}
Roy N, Burgard W, Fox D and Thrun S (1999) Coastal navigation-mobile robot
  navigation with uncertainty in dynamic environments.
\newblock In: \emph{Proceedings 1999 IEEE international conference on robotics
  and automation (Cat. No. 99CH36288C)}, volume~1. IEEE, pp. 35--40.

\bibitem[{Smallwood and Sondik(1973)}]{smallwood1973optimal_pomdp_2}
Smallwood RD and Sondik EJ (1973) The optimal control of partially observable
  markov processes over a finite horizon.
\newblock \emph{Operations research} 21(5): 1071--1088.

\bibitem[{Terras(2012)}]{terras2012harmonic}
Terras A (2012) \emph{Harmonic analysis on symmetric spaces and applications
  II}.
\newblock Springer Science \& Business Media.

\bibitem[{Van Den~Berg et~al.(2011)Van Den~Berg, Abbeel and
  Goldberg}]{van2011lqg_van_2010}
Van Den~Berg J, Abbeel P and Goldberg K (2011) Lqg-mp: Optimized path planning
  for robots with motion uncertainty and imperfect state information.
\newblock \emph{The International Journal of Robotics Research} 30(7):
  895--913.

\bibitem[{Vitus and Tomlin(2011)}]{vitus2011closed_related_to_cc_3}
Vitus MP and Tomlin CJ (2011) Closed-loop belief space planning for linear,
  gaussian systems.
\newblock In: \emph{2011 IEEE International Conference on Robotics and
  Automation}. IEEE, pp. 2152--2159.

\bibitem[{Wang et~al.(2020)Wang, Jasour and Williams}]{wang2020non_cc_extend_2}
Wang A, Jasour A and Williams BC (2020) Non-gaussian chance-constrained
  trajectory planning for autonomous vehicles under agent uncertainty.
\newblock \emph{IEEE Robotics and Automation Letters} 5(4): 6041--6048.

\bibitem[{Zinage et~al.(2023)Zinage, Arul, Manocha and Ghosh}]{zinage20233d}
Zinage V, Arul SH, Manocha D and Ghosh S (2023) 3d-online generalized sensed
  shape expansion: A probabilistically complete motion planner in
  obstacle-cluttered unknown environments.
\newblock \emph{IEEE Robotics and Automation Letters} 8(6): 3334--3341.

\bibitem[{Zinage and Ghosh(2020)}]{zinage2020generalized}
Zinage VV and Ghosh S (2020) Generalized shape expansion-based motion planning
  in three-dimensional obstacle-cluttered environment.
\newblock \emph{Journal of Guidance, Control, and Dynamics} 43(9): 1781--1791.

\end{thebibliography}

\appendix
 \section{Proof of Theorem~\ref{theo:volume}}
 \label{app:volume}


We adopt the eigenvalue decomposition to parameterize $P \in \mathbb{S}^d_{+}$ as
\[P(\Lambda,U)=U^\top \Lambda U,\]
where $\Lambda=\textup{diagonal}(\lambda_1, \dots, \lambda_d)$ and $U$ is a $d\times d$ real orthonormal matrix (i.e., $UU^\top=U^\top U=I_d$). Using this parametrization, referred to as polar parametrization, we can express the volume element $dP$ as
\begin{align*}
dP = \det \theta \;   \prod_{i<j}^d |\lambda_i-\lambda_j| \prod_i^d d\lambda_i, 
\end{align*}
where $\theta_{ij}=\sum_{k=1}^d U_{jk}dU_{ik}$ and
$\det \theta$ is the exterior product $\det(\theta)= \wedge_{i<j} \theta_{i,j}$ \cite{terras2012harmonic, mathai1997jacobians}.
The integral of any function $f:\mathbb{S}^d_{+} \rightarrow \mathbb{R}$ over $\mathbb{S}^d_{+}$ can be computed using polar parametrization as: 
\begin{align}
\nonumber
    &\int_{\mathbb{S}^d_{+}} f(P) dP = (d! \; 2^d)^{-1}  \int_{O(d)} \int_{\mathbb{R}^d_{+}}\bigg(\\
    \label{eq:lambda}
    &  f(P(\Lambda, U))\det \theta \;   \prod_{i<j}^d |\lambda_i-\lambda_j|\bigg) \prod_{i=1}^d d\lambda_i. 
\end{align}
where $O(d)$ is space of $d\times d$ real orthonormal matrices. The factor $(d! \; 2^d)^{-1}$ in \eqref{eq:lambda} resolves the ambiguity of all possible orderings of $\lambda_1, \dots, \lambda_d$ and all possible orientation (multiplication by $+1$ or $-1$) of the columns of $U$.
 The volume of $\mathcal{D}_A$ can be computed by substituting the indicator function of $\mathcal{D}_A$ in place of $f(P(\Lambda,U))$ in \eqref{eq:lambda}. Thus,
\begin{align*}
&\textup{Vol}(\mathcal{D}_A) \nonumber\\
=&\int_{O(d)} \int_{\Lambda\preceq A}\bigg(\det \theta \;   \prod_{i<j}^d |\lambda_i-\lambda_j|\bigg) \prod_{i=1}^d d\lambda_i\nonumber\\
 &\geq\textup{Vol}(\mathcal{D}'_A)\nonumber\\
 =&(d! \; 2^d)^{-1} \left( \int_{O(d)} \det \theta \right) \\
&\left( \int_{0 \leq \lambda_i \leq \underset{k\in[1;d]}{\min}a_k} 
 \prod_{i<j}^d |\lambda_i-\lambda_j| \prod_i^d d\lambda_i\right)
\end{align*}
It is a well-known result (see e.g., \cite[page 71]{muirhead2009aspects} or \cite{chikuse2003statistics} ) that $\int_{O(d)} \det \theta= \frac{2^d \pi^{d^2/2}}{\Gamma_d(d/2)}$, where $\Gamma_d$ is the gamma function. To compute the integral
\begin{align}
\label{eq:integ}
    \int_{0 \leq \lambda_i \leq \underset{k\in[1;d]}{\min}a_k}    \prod_{i<j}^d |\lambda_i-\lambda_j| \; \prod_i^d d\lambda_i,
\end{align}
we use the transformation $y_i=\lambda_i/b$ where $b=\underset{k\in[1;d]}{\min}\;a_k$. Therefore, $\prod_{i=1}^ddy_i=\prod_{i=1}^d\lambda_i/b^d$. Consequently, the integral \eqref{eq:integ} becomes
\begin{align}
    &\int_{0 \leq y_i \leq 1}    \prod_{i<j}^d |y_i-y_j| \; \prod_i^d dy_i\nonumber\\
    &=b^{\frac{d(d+1)}{2}}\int_{0 \leq y_i \leq 1}    \prod_{i<j}^d |y_i-y_j| \; \prod_i^d dy_i\nonumber\\
    &=b^{\frac{d(d+1)}{2}}S_d(1,1,1/2)
\end{align}
where $S_d(\alpha_1,\alpha_2,\alpha_3)$ is the Selberg integral \cite{grafakos1999selberg} given by
\begin{align}
     &S_d(\alpha_1, \alpha_2, \alpha_3)\nonumber\\
     &=\int_0^1 .. \int_0^1 \prod_{i=1}^d t_i^{\alpha_1-1}\left(1-t_i\right)^{\alpha_2-1} \prod_{1 \leq i<j \leq d}\left|t_i-t_j\right|^{2 \alpha_3} d t \nonumber \\ 
     & =\prod_{j=0}^{d-1} \frac{\Gamma(\alpha_1+j \alpha_3) \Gamma(\alpha_2+j \gamma) \Gamma(1+(j+1) \gamma)}{\Gamma(\alpha_1+\alpha_2+(d+j-1) \alpha_3) \Gamma(1+\alpha_3)}\nonumber
\end{align}
which results to \eqref{eq:vol_DA}. Now, for $\beta>1$, we have
\begin{align}
     &\frac{\textup{Vol}(\mathcal{D}_{\beta A})}{\textup{Vol}(\mathcal{D}_A)}\nonumber\\
     &=\frac{\int_{O(d)} \int_{\Lambda\preceq \beta A}\bigg(\det \theta \;   \prod_{i<j}^d |\lambda_i-\lambda_j|\bigg) \prod_{i=1}^d d\lambda_i}{\int_{O(d)} \int_{\Lambda\preceq A}\bigg(\det \theta \;   \prod_{i<j}^d |\lambda_i-\lambda_j|\bigg) \prod_{i=1}^d d\lambda_i}
\end{align}
Substituting $z_i= \lambda_i/\beta$, we have
\begin{align}
    &\frac{\textup{Vol}(\mathcal{D}_{\beta A})}{\textup{Vol}(\mathcal{D}_A)}\nonumber\\
     &=\frac{\beta^{\frac{d(d+1)}{2}}\int_{O(d)} \int_{\Lambda\preceq A}\bigg(\det \theta \;   \prod_{i<j}^d |z_i-z_j|\bigg) \prod_{i=1}^d dz_i}{\int_{O(d)} \int_{\Lambda\preceq A}\bigg(\det \theta \;   \prod_{i<j}^d |\lambda_i-\lambda_j|\bigg) \prod_{i=1}^d d\lambda_i}\nonumber\\
     &=\beta^{\frac{d(d+1)}{2}}
\end{align}

Furthermore, we have the following
\begin{align}
    \frac{\textup{Vol}(\mathcal{D}'_{\beta A})}{\textup{Vol}(\mathcal{D}'_A)}=\frac{ \int_{0 \leq \lambda_i \leq \beta b}    \prod_{i<j}^d |\lambda_i-\lambda_j| \; \prod_i^d d\lambda_i}{ \int_{0 \leq \lambda_i \leq b}    \prod_{i<j}^d |\lambda_i-\lambda_j| \; \prod_i^d d\lambda_i}
    \label{eqn:div}
\end{align}
Consider the change of variable as $b_i=\lambda_i/\beta$. Then, \eqref{eqn:div} becomes
\begin{align}
        \frac{\textup{Vol}(\mathcal{D}'_{\beta A})}{\textup{Vol}(\mathcal{D}'_A)}&=\frac{\beta^{\frac{d(d+1)}{2}} \int_{0 \leq b_i \leq  a_i}    \prod_{i<j}^d |b_i-b_j| \; \prod_i^d db_i}{ \int_{0 \leq \lambda_i \leq a_i}    \prod_{i<j}^d |\lambda_i-\lambda_j| \; \prod_i^d d\lambda_i}\nonumber\\
        &=\beta^{\frac{d(d+1)}{2}}=\frac{\textup{Vol}(\mathcal{D}_{\beta A})}{\textup{Vol}(\mathcal{D}_A)}
        \label{eqn:div_beta}
\end{align}
Subtracting 1 from both sides of \eqref{eqn:div_beta}, we get
\begin{align}
    &\frac{\textup{Vol}(\mathcal{D}_{\beta A})-\textup{Vol}(\mathcal{D}_{ A})}{\textup{Vol}(\mathcal{D}_{ A})}=\frac{\textup{Vol}(\mathcal{D}'_{\beta A})-\textup{Vol}(\mathcal{D}'_{ A})}{\textup{Vol}(\mathcal{D}'_{ A})}\nonumber\\
    &\geq \textup{Vol}(\mathcal{D}'_{\beta A})-\textup{Vol}(\mathcal{D}'_{ A})\nonumber\\
    &=(\beta^{\frac{d(d+1)}{2}}-1)\textup{Vol}(\mathcal{D}'_{ A})\nonumber\\
    &\geq (\beta^{\frac{d(d+1)}{2}}-1) b^{\frac{d(d+1)}{2}}S_d(1,1,1/2)
\end{align}
and hence the result follows.

\section{Proof of Lemma~\ref{lemma:diag} }
\label{app:diag}
$\textup{Vol}(\mathcal{D}_P)$ can be computed by substituting the indicator function of $\mathcal{D}_P$ in place of $f(P(\Lambda,U))$ in \eqref{eq:lambda}. If we introduce a transformation of variables from $U$ to $U'=U V^\top$, the Jacobian of the transformation is  $1$, and the integral becomes  the $\textup{vol}(\mathcal{D}_\Sigma)$ in the $(\theta', P)$ coordinate system. $\qed$ 

\section{Proof of Lemma \ref{lemma:chain_col_free_lossless}\label{appendix:chain_col_free_lossless}}
To show that the transition from $(x_{n,k}, P'_{n,k})$ to $(x_{n,k+1}, P'_{n,k+1})$ is collision-free, it is sufficient to show that
\begin{equation}
\label{eq:claim2_1}
    \mathcal{E}_{\chi^2}(x_{n,k+1}, P'_{n,k}+\ell^L_n W) \subset\mathcal{E}_{\chi^2}(x_{n,k}, P_{n,k}).
\end{equation}
In what follows, we derive a sufficient condition for $\ell^L_n$ to satisfy \eqref{eq:claim2_1}. Suppose that $\ell^L_n$ is small enough so that
\begin{equation}
\label{eq:claim2_2}
   P'_{n,k}+\ell^L_n W \preceq P'_{n,k}+3\ell^L_n W \preceq \left(1-\frac{\delta^L_n}{2}\right)^2 P_{n,k}
\end{equation}
holds. \eqref{eq:claim2_2} is equivalent to
\begin{align*}
    (1-\delta^L_n)^2 P_{n,k}+\ell^L_n W &\preceq(1-\delta^L_n)^2 P_{n,k}+3\ell^L_n W\\
    &\preceq \left(1-\frac{\delta^L_n}{2}\right)^2 P_{n,k}
\end{align*}
or
\begin{equation}
    \label{eq:claim2_3}
   3 \ell^L_n W \preceq \delta^L_n \left(1-\frac{3}{4}\delta^L_n\right) P_{n,k}.
\end{equation}
which trivially holds as $    \ell^L_n \leq  \delta^L_n \left(1-\frac{3}{4}\delta^L_n\right) \frac{\rho}{3 \bar{\sigma}(W)}$.
In order for \eqref{eq:claim2_2} to imply \eqref{eq:claim2_1}, we further require the distance $\|x_{n,k+1}-x_{n,k}\|$ between the centers of two ellipses in \eqref{eq:claim2_1} is less than or equal to the difference between semi-minor axis lengths of ellipses characterized by $P_{n,k}$ and $(1-\frac{\delta^L_n}{2})^2 P_{n,k}$ computed as $\frac{\delta^L_n}{2}\sqrt{\chi^2\underline{\sigma}(P_{n,k})}$.
Notice that $
\|x_{n,k+1}-x_{n,k}\|\leq \ell^L_n \leq 3\ell^L_n$.
On the other hand, the difference between the semi-minor axis lengths is greater than $\frac{\delta^L_n}{2}\sqrt{\chi^2\rho}$. Therefore, if
\begin{equation}
    \label{eq:claim2_5}
    \ell^L_n \leq  \frac{\delta^L_n}{6}\sqrt{\chi^2\rho}
\end{equation}
then \eqref{eq:claim2_2} implies \eqref{eq:claim2_1}. To summarize, since $\ell^L_n$ simultaneously satisfies \eqref{eq:claim2_2} and \eqref{eq:claim2_5}, \eqref{eq:claim2_1} holds. $\qed$

\section{Proof of Lemma \ref{lemma:collision_free_lossless}\label{appendix:collision_free_lossless}}
\begin{figure}[ht!]
\centering
\includegraphics[width=1\columnwidth]{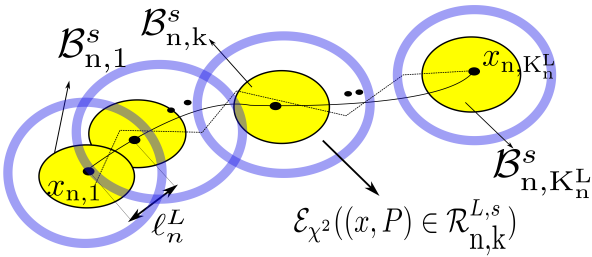}
\caption{Covering of collision-free and lossless chain $\gamma''_n$ with balls $\mathcal{B}^s_{n,k}$ of radius $\frac{(1-(\theta^1_{n,k})^2)}{2}\ell^L_n$. The event $E_{n,k}$ is the event that a sampled point $(x,P)\in\mathcal{R}^{L,s}_{n,k}$.}
\label{fig:covering_balls_lossless}
\end{figure}
Lets define $\mathcal{E}^{\text{out}}_{n,k}:=\mathcal{E}_{\chi^2}(x_{n,k},{(\theta^0_{n,k})^2}P''_{n,k})$, it is easy to verify that the minimum distance $\mathcal{E}^{\text{out}}_{n,k}$ and $\mathcal{E}_{\chi^2}(x_{n,k},P''_{n,k})$ is $(\ell^L_n)^4$. On the other hand, any translation of the center of ellipsoid  $\mathcal{E}^{\text{out}}_{n,k}$ from $x_{n,k}$ to some  $x\in\mathcal{B}(x_{n,k},(\ell^L_n)^4)$, will linearly translate the ellipsoid $\mathcal{E}_{\chi^2}\left(x_{n,k},P\right)$ by a maximum of distance $(\ell^L_n)^4$.
Thus, after the translation the ellipse stays inside $\mathcal{E}(x_{n,k},P''_{n,k}) \subset \mathcal{E}(x_{n,k}, P_{n,k}) $ and \eqref{eq:init_in} holds. 
As the first step to prove \eqref{eq:final_in_lossless}, we stress that  
\begin{align*}
 &||x_2-x_1|| \leq \\
 &||x_{n,k}-x_1 ||+ ||x_{n,k+1}- x_{n,k} ||+||x_2-x_{n,k+1}||\\
 &\leq (\ell^L_n)^4 + (\ell^L_n)^4 + (\ell^L_n)^4 \leq \ell^L_n + \ell^L_n + \ell^L_n = 3\ell^L_n.  
\end{align*}
Thus, we have
\begin{align} \nonumber
     P_1+||x_2-x_1||W &\preceq  P''_{n,k}+||x_2-x_1||W \\ \nonumber
     &\preceq P'_{n,k}+||x_2-x_1||W \preceq 
     P'_{n,k}+ 3 \ell^L_n W\\ \label{eq:final_ell}
     &\preceq \left(1-\frac{\delta^L_n}{2}\right)^2 P_{n,k},
\end{align}
where the last inequality is shown earlier as \eqref{eq:claim2_2}. From 
\eqref{eq:final_ell}, we have
\begin{align*}
&\mathcal{E}(x_{n,k}, P_1+||x_2-x_1||W)
\subseteq \mathcal{E}\left(x_{n,k},  \left(1-\frac{\delta^L_n}{2}\right)^2 P_{n,k}\right). 
\end{align*}
The minimum distance between $\mathcal{E}(x_{n,k}, P_{n,k})$ and $\mathcal{E}\left(x_{n,k},  \left(1-\frac{\delta^L_n}{2}\right)^2 P_{n,k}\right)$ is  
$\frac{\delta^L_n}{2}\sqrt{\chi^2\underline{\sigma}(P_{n,k})}$, Thus after linear translating $\mathcal{E}\left(x_{n,k},  \left(1-\frac{\delta^L_n}{2}\right)^2 P_{n,k}\right)$  for  
$||x_2-x_{n,k}||\leq \|x_2-x_{n,k+1}\|+\|x_{n,k}-x_{n,k+1}\| \leq 2\ell^L_n \leq \frac{\delta^L_n}{2} \sqrt{\chi^2\rho} \leq \frac{\delta^L_n}{2}\sqrt{\chi^2 \underline{\sigma}(P_{n,k})}$, the resultant ellipse $\mathcal{E}_{\chi^2}(x_2, 
\left(1-\frac{\delta^L_n}{2}\right)^2 P_{n,k})$, and subsequently $\mathcal{E}_{\chi^2}\left(x_2, P_1+||x_2-x_1||W\right)$ stays inside 
$\mathcal{E}_{\chi^2} \left(x_{n,k},P_{n,k}\right)$.$\qed$
\section{Proof of Lemma \ref{lemma:transition_is_lossless}\label{appendix:transition_is_lossless}}
We know that the chain $(x_{n,k},P''_{n,k})$ for all $k\in[1;K^L_n]$ is lossless and collision-free. Therefore,
\begin{align}
    P''_{n,k+1}\preceq P''_{n,k}+\ell^L_nW
    \label{eqn:lossless}
\end{align}
for all $k\in[1;K^L_{n}-1]$. Further, since sampled point $(x_2,P_2)\in\mathcal{R}^{L,s}_{n,k+1}$, the following is true
\begin{align} \nonumber
    &(\theta^1_{n,k+1})^2P''_{n,k+1}\preceq P_2\preceq {(\theta^2_{n,k+1}})^2P''_{n,k+1}\nonumber\\
    &\implies  \frac{P_2}{(\theta^2_{n,k+1})^2}\preceq P''_{n,k+1}\preceq  \frac{P_2}{(\theta^1_{n,k+1})^2}
    \label{eqn:ineq_mix}
\end{align}
Using \eqref{eqn:lossless} and \eqref{eqn:ineq_mix}, $\frac{P_2}{(\theta^2_{n,k+1})^2}\preceq P''_{n,k}+\ell^L_nW$. 
Further, using the fact that $\theta^1_{n,k}=\theta^2_{n,k+1}$ for $k\in[1;K^L_n-1]$ implies the following
\begin{align} \nonumber
      {P_2}&\preceq {(\theta^2_{n,k+1})^2}P''_{n,k}+{(\theta^2_{n,k+1})^2}\ell^L_nW\\
      &={(\theta^1_{n,k})^2}P''_{n,k}+{(\theta^1_{n,k})^2}\ell^L_nW \nonumber\\
      &\preceq P_1+{(\theta^1_{n,k})^2}\ell^L_nW\preceq P_1+\frac{((\theta^1_{n,k+1})^2+(\theta^1_{n,k})^2)}{2}\ell^L_nW\nonumber
\end{align}
Using the fact that $\|x_{n,k+1}-x_{n,k}\|=\ell^L_n$, $(x_1,P_1)\in\mathcal{R}^{L,s}_{n,k}$ and $(x_2,P_2)\in\mathcal{R}^{L,s}_{n,k+1}$, the minimum distance between $x_2$ and $x_1$ is given as
\begin{align}
    \|x_2-x_1\|&\geq\ell^L_n-\frac{(1-(\theta^1_{n,k})^2)}{2}(\ell^L_n)^4\nonumber\\
    &-\frac{(1-(\theta^1_{n,k+1})^2)}{2}(\ell^L_n)^4\nonumber\\
    &\geq\frac{(\theta^1_{n,k+1})^2+(\theta^1_{n,k})^2}{2}\ell^L_n.\nonumber
\end{align}
Consequently, $P_2\preceq P_1+\|x_2-x_1\|W$.
Therefore, $P_2\preceq P_1+\|x_2-x_1\|W $. Further, since $\mathcal{R}^{L,s}_{n,k+1}\subset\mathcal{R}^L_{n,k+1}$ by construction, the edge between two sampled points $(x_1,P_1)\in\mathcal{R}^{L,s}_{n,k}$ and $(x_2,P_2)\in\mathcal{R}^{L,s}_{n,k+1}$ for all $k\in[1;K^L_n-1]$ is finitely lossless and collision-free.$\qed$
\section{Proof of Lemma \ref{lemma:distance_smaller_than_ed_min_lossless}\label{appendix:distance_smaller_lossless}}
{ We first note that the maximum distance between $x_2$ and $x_1$ is
\begin{align*}
    \|x_2-x_1\| &\leq \ell^L_n+\frac{(1-(\theta^1_{n,k})^2)}{2}(\ell^L_n)^4+\nonumber\\
    &\frac{(1-(\theta^1_{n,k+1})^2)}{2}(\ell^L_n)^4\nonumber\\
    &\leq\left(2-\frac{(\theta^1_{n,k+1})^2+(\theta^1_{n,k})^2}{2}\right)\ell^L_n\nonumber \leq 2 \ell^L_n.
\end{align*}
}
Therefore, we have $\hat{D}(b_1,b_2)=\|x_2-x_1\|\leq 2\ell^L_n:=D_\text{min}$.
\section{Proof of Lemma \ref{lemma:event_E_n_for_lossless}\label{appedix:event_en_lossless}}
Then, due to the uniform distribution, $    \mathbb{P}(\{x_i\in\mathcal{B}(x_{n,k},(\ell^L_n)^4)\})$ is given by
\begin{align} \nonumber
    \mathbb{P}(\{x_i\in\mathcal{B}(x_{n,k},(\ell^L_n)^4)\})&=\frac{\text{vol}(\mathcal{B}(x_{n,k},(\ell^L_n)^4))}{\text{vol}(\mathcal{X}_{\text{free}})}\nonumber\\
    &=\frac{\tau_d(\ell^L)^{4d}_n}{\mathcal{V}_{\mathcal{X}}}\geq\frac{\tau_d(\delta^L_n)^{4d}(h^L)^{4d}}{\mathcal{V}_{\mathcal{X}}},\nonumber
\end{align}
where $\mathcal{V}_{\mathcal{X} }:=\text{vol}(\mathcal{X}_{\text{free}})$. It
is straightforward to verify that
\begin{align} \nonumber
    \mathbb{P}(\{x_i\in\mathcal{B}^s(x_{n,k},\ell^L_n)\})\!&=\!c_{n,k} \mathbb{P}(\{x_i\in\mathcal{B}(x_{n,k},(\ell^L_n)^4)\}) \\ \nonumber
    &\geq\frac{ \tau_d(\delta^L_n)^{6d}(h^L)^{6d}}{\mathcal{V}_{\mathcal{X}}\chi^\frac{d}{2}\rho^{\frac{d}{4}} 2^d}=g^L_1 (\delta^L_n)^{6d}
\end{align}
where $g^L_1:=\frac{ \tau_d (h^L)^{6d}}{\mathcal{V}_{\mathcal{X}}\chi^\frac{d}{2}\rho^{\frac{d}{4}} 2^d}$.
We have  $\textup{Tr}(P) \geq 16 (\theta^1_{n,k})^2  \rho d \geq \rho d$ for all $P \in \mathcal{D}_{n, k}$.
On the other hand, $\textup{Tr}(P)\leq (\theta^2_{n,k})^2 \textup{Tr}(P''_{n,k}) \leq  (\theta^2_{n,k})^2 R \leq R $. Hence, $\mathcal{D}_{n,k} \subset \mathcal{R}_{ [(\theta^1_{n,k})^2  16\rho d, (\theta^2_{n,k})^2 R ]} \subset \mathcal{R}_{ [ \rho d, R ]}$.

From the definition of uniform sampling, we have 
\begin{align*}
     &\mathbb{P}(P_i\in\mathcal{D}_{n,k})
     {\geq}\frac{\text{vol}(\mathcal{D}_{n,k}\cap \mathcal{D}_{[\rho d, R]})}{\text{vol}(\mathcal{D}_{[\rho d ,R]})} = \frac{\text{vol}(\mathcal{D}_{n,k})}{\text{vol}(\mathcal{D}_{[\rho d,R]})} \\
    &=\frac{\text{vol}(\mathcal{D}_{(\theta^2_{n,k})^2P''_{n,k}}) - \text{vol}(\mathcal{D}_{(\theta^1_{n,k})^2P''_{n,k}})}{\text{vol}(\mathcal{R}_{R}) - \text{vol}(\mathcal{R}_{\rho d})}\\
    &\geq\frac{V_dS_d(1,1,1/2)b^{\frac{d(d+1)}{2}}\bigg[ (\theta^2_{n,k})^{d(d+1)}- (\theta^1_{n,k})^{d(d+1)} \bigg]}{\left(\frac{2}{d(d+1)}\right)V_r \bigg[R^{\frac{d(d+1)}{2}}- (\rho d)^{\frac{d(d+1)}{2}} \bigg]}.
\end{align*}
where $b=\underset{i\in[1;d]}{\min}\;\lambda_i(P''_{n,k})$. Now, we have
\begin{align*}
    &(\theta^2_{n,k})^{d(d+1)}-(\theta^1_{n,k})^{d(d+1)} \nonumber\\
    &= (\theta^2_{n,k} - \theta^1_{n,k}) \sum_{j=0}^{d(d+1)-1} (\theta^2_{n,k})^{d-1-j}  (\theta^1_{n,k})^{j}\\
        &\geq d\frac{(\theta^2_{n}-\theta^1_{n})}{K^L_n} (\theta^1_{n})^{d(d+1)-1}\nonumber\\
        &\geq d\frac{(\theta^2_{n}-\theta^1_{n})}{2K^L_n}(\theta^2_{n}+\theta^1_{n}) (\theta^1_{n})^{d(d+1)-1}\\
    &\geq d\frac{\ell^L_n}{2\ell^\star}(\delta^L_n)^2 \left(1-\frac{(\delta_n^L)^2}{9}\right)^{d(d+1)-1}\nonumber\\
    &\geq\frac{h^Ld}{\ell^\star}\left(\frac{\delta^L_n}{3}\right)^{2d(d+1)}.\nonumber
\end{align*}
where in the last step we use $\delta^L_n\leq 1$. 
In sum, it can be deduced
\begin{align}
    \mathbb{P}(P_i\in\mathcal{D}_{n,k}) \geq g^L_2 (\delta^L_n)^{2d(d+1)},
\end{align}
where $g^L_2 = \frac{V_dV_r^{-1}S_d(1,1,1/2)2^{-1}d^2(d+1) h^L(16\rho)^{\frac{d(d+1)}{2}}}{\ell^\star 3^{2d(d+1)}} \bigg[R^{\frac{d(d+1)}{2}}- (\rho d)^{\frac{d(d+1)}{2}} \bigg]^{-1} $ is a constant. 

Lets define the event $E^i_{n,k}$ is as the event that the sampled belief $b_i=(x_i,P_i)$ belongs to $\mathcal{R}^s_{n,k}$. Then, $E_{n,k}= \cup_{i=1}^n E^i_{n,k}$. The following lemma shows that if the $\gamma$ is greater than a certain positive threshold, then the probability that event $E_n\triangleq E_{n,1}\cap E_{n,2}\dots \cap E_{n,K^L_n}$ occurs equals to one as $n$ approaches infinity.
The probability of event $E^c_{n,k}$ for all $k\in[1;K^L_n]$ is given as follows:
\begin{align}
   &\mathbb{P}\left(E^c_{n,k}\right)
   = \prod_{i=1}^{n}(1-\mathbb{P}\left(E^i_{n,k}\right))\nonumber\\
   &\leq \left(1- \mathbb{P}(P\in\mathcal{D}_{n,k}) \;     \mathbb{P}(x\in\mathcal{B}_{n,k})\right)^n\leq \nonumber\\ \nonumber
   &\leq\left[1- g^L_1 (\delta^L_n)^{6d} g^L_2 (\delta^L_n)^{2d(d+1)} \right]^n \nonumber\\
   &= \left[1- g^L_1 g^L_2 (\delta^L_n)^{d(2d+8)} \right]^n.\nonumber
    \end{align}
Using the fact that $(1-x)\leq e^{-x}$ for $x\in(0,1)$, and substituting $\delta^L_n$, for sufficiently large $n$, we have
\begin{align} \nonumber
 \mathbb{P}\left(E^c_{n,k}\right)&\leq \left(1-g^L_1 g^L_2
 {\gamma^{d(2d+8)}\frac{\text{log}n}{n}} \right)^{n} \leq
  n^{- g^L_1 g^L_2 \gamma^{d(2d+8)}}.\nonumber
\end{align}
Now, the event $E_n^c=\bigcup_{k=1}^n E^c_{n,k}$ is upper bounded as follows:
\begin{align}
  \mathbb{P}\left(E^c_{n}\right)& \!=\!\mathbb{P}\left(\bigcup_{k=1}^{K^L_n} E^c_{n,k}\right) 
 \! \leq \! \sum_{k=1}^{K^L_n}\mathbb{P}\left( E^c_{n,k}\right)\!\leq \! K^L_n n^{- g^L_1 g^L_2 \gamma^{d(2d+8)}}.  \nonumber
\end{align}
where $K^L_n=\lfloor\frac{\ell^\star}{\ell^L_n}\rfloor\leq\frac{\ell^\star}{\ell^L_n}$.
Since $\delta^L_n h^L\leq\ell^L_n$, we have
\begin{align} \nonumber
    &\mathbb{P}\left(E^c_{n}\right)
    \leq \frac{\ell^\star}{\delta^L_n h^L}  n^{- g^L_1 g^L_2 \gamma^{d(2d+8)}}  \\ \label{eq:P_c_lossless} 
     &= \frac{ \ell^\star}{\delta^L_n h^L} (\log n)^{-\frac{1}{d(2d+8)}} n^{- g^L_1 g^L_2 \gamma^{d(2d+8)}+ \frac{1}{d(2d+8)}}. 
\end{align}
If the power of $n$ in \eqref{eq:P_c_lossless} is less than $-1$, which is equivalent to
\[
\gamma>\left(\frac{d(2d+8)+1}{g^L_1 g^L_2 d(2d+8)} \right)^\frac{1}{d(2d+8)},
\]
we have $\sum_{n=1}^\infty  \mathbb{P}\left(E^c_{n}\right)<\infty$. Consequently, 
$\underset{n\rightarrow\infty}{\lim}\mathbb{P}\left(E^c_{n}\right)=1$,
by Borel Cantelli lemma \cite{grimmett2020probability} which completes the proof. $\qed$

\section{Proof of Lemma \ref{lemma:arbritrarily_close_for_lossless}\label{appendix:arbitr_close_lossless}}
Define $I_{n,k}$ as follows:
\begin{align}
    I_{n, k}:= \begin{cases}1, & \text { if }\mathcal{R}^{L,\beta s}_{n,k} \cap V^{\mathrm{Lossless \;IG-PRM}^{*}}=\emptyset, \\ 0, & \text { otherwise. }\end{cases}\nonumber
\end{align}

Define $M_n=\sum_{k=1}^{K^L_n} I_{n,k}$, and assume the event $\{M_n\leq \alpha K^L_n\}$ has occurred which means that $\alpha$ fraction of the $K^L_n$ points be such that the sampled points do not belong to any $\mathcal{R}^{L,\beta s}_{n,k}$. 
If the sampled $(x,P)$ is not inside $\mathcal{R}^{L,\beta s}_{n,k}$ but it is inside $\mathcal{R}^{L,s}_{n,k}$, we have
\begin{align*}
     \|x-x_{n,k}\|\leq &  \frac{(1-(\theta^1_{n,k})^2)}{2}(\ell^L_n)^4
     \leq  \frac{1}{2} \ell^L_n,\\
    \|P-P''_{n,k}\|_F &\leq  \|(\theta^1_{n,k})^2 P''_{n,k}-P''_{n,k}\|_F \\
    &\leq \|(\theta^1_n)^2 P''_{n,k}-P''_{n,k}\|_F\leq \frac{16 \bar{\rho} \sqrt{d}}{\rho}  \ell^L_n,
\end{align*}
which yields $\hat{\mathcal{D}}((x,P), (x_{n,k},P''_{n,k})) \leq  c \ell^L_n$, where $c:= \frac{1}{2}+ \frac{16 \bar{\rho} \sqrt{d}}{\rho}$.
Similarly, if the sampled $(x,P)$ is inside $\mathcal{R}^{L,\beta s}_{n,k}$ we have 
\begin{align*}
    & \|x-x_{n,k}\|\leq  \beta\frac{(1-(\theta^1_{n,k})^2)}{2}(\ell^L_n)^4
     \leq  \frac{1}{2} \beta\ell^L_n,\nonumber\\
    &\|P-P''_{n,k}\|_F \nonumber\\
    &\leq  \|(\beta\theta^1_{n,k}+(1-\beta)\theta^2_{n,k})^2 P''_{n,k}-(\theta^2_{n,k})^2P''_{n,k}\|_F\\
    &+\|(\theta^2_{n,k})^2P''_{n,k}-P''_{n,k}\|\\
    &\leq 2\|(\beta\theta^1_{n,k}+(1-\beta)\theta^2_{n,k}) P''_{n,k}-\theta^2_{n,k}P''_{n,k}\|_F\\
    &+\|(\theta^2_{n,k})^2P''_{n,k}-P''_{n,k}\|\\
    & \leq 2\beta(\theta^2_{n,k}- \theta^1_{n,k}) \|P''_{n,k}\|_F +\|(\theta^2_{n,k})^2P''_{n,k}-P''_{n,k}\|\\
    &\leq 2\beta\Delta \bar{\rho} \sqrt{d}+(\delta^L_n)^2\bar{\rho} \sqrt{d}
\end{align*}
which yields $\hat{\mathcal{D}}((x,P), (x_{n,k},P''_{n,k})) \leq 2\beta\Delta \bar{\rho} \sqrt{d}+(\delta^L_n)^2\bar{\rho} \sqrt{d}$. If we define $c_1:=\beta \bar{\rho} \sqrt{d}$ and $c_2=\bar{\rho} \sqrt{d}$ (which is bounded), we have
\begin{align}
\left\|\gamma^p_{n}-\gamma''_{n}\right\|_{TV}& \leq  \sum_{k=1}^{K^L_n}  \hat{\mathcal{D}}((x,P), (x_{n,k},P''_{n,k})\nonumber\\
&\leq K^L_n (\alpha c \ell^L_n + (1-\alpha) \beta c_1 \Delta+c_2(\delta^L_n)^2) \nonumber\\
&\leq c\alpha L+c_1\beta (\theta^2_n-\theta^1_n)+c_2\delta^L_nL/h^L,
\end{align}
where $L=\underset{n}{\sup}\;\gamma''_n$. The fact that the bounded variation $\|\gamma^p_n-\gamma''_n\|_{TV}$ is upper bounded by $c\alpha L+c_1\beta (\theta^2_n-\theta^1_n)+c_2\delta^L_nL/h^L$ implies that
\begin{align}
&\left\{M_{n} \leq \alpha K_{n}\right\} \subseteq\nonumber\\
&\left\{\left\|\gamma^p_{n}-\gamma''_{n}\right\|_{TV} \leq c\alpha L+c_1\beta (\theta^2_n-\theta^1_n)+c_2\delta^L_nL/h^L\right\}\label{eqn:MclessthatalphaKn}
\end{align}
Taking the complement of both sides of \eqref{eqn:MclessthatalphaKn} and using the monotonicity of probability measures,
\begin{align}
&\mathbb{P}\left(\left\{\left\|\gamma^p_{n}-\gamma''_{n}\right\|_{TV}>c\alpha L+c_1\beta (\theta^2_n-\theta^1_n)+c_2\delta^L_nL/h^L\right\}\right)\nonumber\\
&\leq \mathbb{P}\left(\left\{M_{n} \geq \alpha K_{n}\right\}\right).
\label{eq:complement}
\end{align}
Since \eqref{eqn:complement} is true for any $\alpha,\beta\in(0,1)$, it remains to show that $\mathbb{P}\left(\left\{M_{n} \geq \alpha K_{n}\right\}\right) $ is finite. {
Let's denote $\ell_{\beta n,k}:=\frac{\beta(1-(\theta_{n,k}^1)^2)}{2}(\ell^L_n)^4)$, $\theta^1_{\beta n,k}:= (\beta\theta^1_{n,k}+(1-\beta)\theta^2_{n,k}) $ and $\mathcal{D}^s_{n,k}=\mathcal{D}_{[(\theta^1_{\beta n,k})^2 P''_{n,k},(\theta^2_{ n,k})^2 P''_{n,k}]}$ for simplicity of notation. Then, expected value of $I_{n,m}$ can be computed as 
\begin{align}
    &\mathbb{E}[I_{n,m}]\nonumber\\
    &= \mathbb{P}(\{I_{n,m}=1\}) \nonumber\\
    &=\left( 1- \mathbb{P}(x\in\mathcal{B}(x,\ell_{\beta n,k})) \times  \mathbb{P}(P\in\mathcal{D}^s_{n,k}) \right)^n\nonumber\\
&\leq     \left(1-g^L_1 g^L_2 \beta^{d+1} (\delta^L_n)^{d(2d+8)}\right)^{n}\nonumber\\
    &\leq    \text{exp} \left(-g^L_1 g^L_2 \beta^{d+1} \gamma^{d(2d+8)} \frac{\log n}{n}  n\right)\nonumber\\
    &= n^{\frac{-\beta^{d+1}g_1^Lg_2^L(d(2d+8)+1)}{d(2d+8)}}. \nonumber
\end{align}

Thus, $\mathbb{E}[M_n]=\sum_{m=1}^{K^L_n} \mathbb{E}[I_{n,m}]= K^L_n n^{-\beta ^{d+1}}$. By Markov’s inequality, it follows that
\begin{align}
    \mathbb{P}\left(\left\{M_{n} \geq \alpha K_{n}\right\}\right) \leq \frac{\mathbb{E}[M_n]}{\alpha K^L_n} \leq \frac{K^L_n n^{-\beta ^{d+1}}}{\alpha K^L_n} = \frac{ n^{-\beta ^{d+1}}}{\alpha}.
    \label{eqn:P_Mn}
\end{align}
it is easy to verify that for fixed $\alpha$, \eqref{eqn:P_Mn} tends to $0$ as $n$ tends to $\infty$.
}
Since this argument holds for all $\alpha, \beta$, using the fact that \eqref{eqn:complement} holds and $\underset{n\rightarrow\infty}{\lim}\;\delta^L_n=0$, it can be concluded that for all $\epsilon>0$, $
\mathbb{P}\left(\left\{\left\|\gamma^p_{n}-\gamma''_{n}\right\|_{TV}>\epsilon\right\}\right)<\infty
$.
Finally, by the Borel-Cantelli lemma, $\mathbb{P}\left(\left\{\underset{n\rightarrow\infty}{\lim}\;\left\|\gamma^p_{n}-\gamma''_{n}\right\|_{TV}=0\right\}\right)=1$.
\section{Proof of Lemma \ref{lemma:cont_final_lossless}\label{appendix:cont_final_lossless}}

Since $\gamma''_n$ is a lossless refinement of $\gamma'_n$, we have $c(\gamma''_n)=c(\gamma'_n)$ for each $n\in\mathbb{N}$.
Therefore,
\begin{align}
    \lim_{n\rightarrow\infty} c(\gamma''_n)=c^\star
    \label{eqn:gamma_dd_c_star}
\end{align}
Using \eqref{eqn:gamma_dd_c_star}, the fact that $\mathbb{P}\left(\left\{\underset{n\rightarrow\infty}{\lim}\;\left\|\gamma^p_{n}-\gamma''_{n}\right\|_{TV}=0\right\}\right)=1$ and continuity of the chain, we can conclude that
\begin{align}
       \mathbb{P}(\{\lim_{n\rightarrow\infty} c(\gamma^p_n)=c^\star\})=1 
\end{align}

\end{document}